\documentclass[11pt]{article}
\usepackage[utf8]{inputenc}
\usepackage{amsfonts,latexsym,amsmath,amsthm,amssymb,amscd,euscript,mathrsfs,graphicx}
\usepackage{framed}
\usepackage{fullpage}
\usepackage{color}
\usepackage[shortlabels]{enumitem}
\usepackage[obeyFinal,textsize=scriptsize,shadow]{todonotes}
\usepackage{tikz}
\usetikzlibrary{matrix}
\usepackage{hyperref}
\usepackage{marginnote}
\usepackage{algorithm}
\usepackage{algpseudocode}

\usepackage{multirow} 

\newtheorem{theorem}{Theorem}
\newtheorem{proposition}[theorem]{Proposition}
\newtheorem{lemma}[theorem]{Lemma}
\theoremstyle{definition}

\theoremstyle{remark}

\newcommand{\BE}{\mathbb E}

\newcommand{\BP}{\mathbb P}

\newcommand{\BR}{\mathbb R}

\newcommand{\eps}{\varepsilon}
\newcommand{\bo}{\boldsymbol}

\title{Almost Tight Approximation Algorithms for Explainable Clustering}
\author{Hossein Esfandiari\thanks{Google Research. Email: \texttt{esfandiari@google.com}} , Vahab Mirrokni\thanks{Google Research. Email: \texttt{mirrokni@google.com}} , Shyam Narayanan\thanks{MIT. Work done as an intern at Google Research. Email: \texttt{shyamsn@mit.edu}}}
\date{\today}

\begin{document}

\maketitle

\begin{abstract}
    Recently, due to an increasing interest for transparency in artificial intelligence, several methods of explainable machine learning have been developed with the simultaneous goal of accuracy and interpretability by humans. In this paper, we study a recent framework of explainable clustering first suggested by Dasgupta et al.~\cite{dasgupta2020explainable}. Specifically, we focus on the $k$-means and $k$-medians problems and provide nearly tight upper and lower bounds.
    
    First, we provide an $O(\log k \log \log k)$-approximation algorithm for explainable $k$-medians, improving on the best known algorithm of $O(k)$~\cite{dasgupta2020explainable} and nearly matching the known $\Omega(\log k)$ lower bound~\cite{dasgupta2020explainable}. In addition, in low-dimensional spaces $d \ll \log k$, we show that our algorithm also provides an $O(d \log^2 d)$-approximate solution for explainable $k$-medians. This improves over the best known bound of $O(d \log k)$ for low dimensions~\cite{laber2021explainable}, and is a constant for constant dimensional spaces. To complement this, we show a nearly matching $\Omega(d)$ lower bound. Next, we study the $k$-means problem in this context and provide an $O(k \log k)$-approximation algorithm for explainable $k$-means, improving over the $O(k^2)$ bound of Dasgupta et al. and the $O(d k \log k)$ bound of \cite{laber2021explainable}. To complement this we provide an almost tight $\Omega(k)$ lower bound, improving over the $\Omega(\log k)$ lower bound of Dasgupta et al. 
    Given an approximate solution to the classic $k$-means and $k$-medians, our algorithm for $k$-medians runs in time $O(kd \log^2 k )$ and our algorithm for $k$-means runs in time $ O(k^2 d)$.
\end{abstract}

\section{Introduction}
Clustering is one of the most fundamental optimization techniques that lies at the heart of many applications in machine learning and data mining. Clustering techniques are vastly used for data classification in unsupervised learning and semi-supervised learning, data compression and representation, and even data visualization. As a result, many powerful techniques have been developed for data clustering over the past decades. However, in the past few years due to an increasing demand for transparency, people look with doubt at clusterings, or more generally learning models, that are not interpretable by humans. Consequently, there is an increasing demand to ``\emph{stop explaining black box machine learning models for high stakes decisions and use interpretable models instead}''~\cite{rudin2019stop}. 

With this motivation in mind, we study an easily interpretable and powerful clustering framework suggested by Dasgupta et al.~\cite{dasgupta2020explainable} called \emph{explainable clustering}. This framework is based on decomposing the space of the points using a decision tree where each node separates two clusters via a simple comparison based on one of the dimensions of the space. Decision trees are known as simple and popular explainable models~\cite{molnar2019interpretable,murdoch2019interpretable}.
In this framework, we evaluate our algorithms by the ratio of the cost of the explainable clustering algorithm to an \emph{optimal non-explainable} clustering algorithm. This has also been referred to as the \emph{price of explainability}, since it measures the required cost blowup to guarantee that the clustering is interpretable~\cite{dasgupta2020explainable}.

In this work, we provide almost optimal explainable algorithms for $k$-means clustering and $k$-medians clustering. These clustering problems are central in data analysis and modern machine learning with several applications in mining massive datasets. $k$-means clustering is defined as follows: Given a dataset of $n$ points where each data element is represented by a vector of real-valued features, the goal is to find $k$ representative vectors, called \emph{centers}, such that the sum of squared distances from each input vector to the closest center is minimized. Similarly, in $k$-medians clustering, the goal is to minimize the sum of distances to the closest centers. 
$k$-means clustering and $k$-medians clustering have become essential building blocks for unveiling hidden patterns and extracting information in datasets, especially in the unsupervised clustering contexts where supervised machine learning cannot be applied, or little is known about the data or when the dataset is massive and hence the competitive supervised methods become impractical.

We first study the explainable $k$-medians clustering problem. As our first result, we develop an $O(\log k \log \log k)$-approximation algorithm for this problem. This improves nearly exponentially over the previous $O(k)$-approximation algorithm of Dasgupta et al. We provide an example in Appendix \ref{sec:failure} for which the algorithm of Dasgupta et al. achieves a $\Theta(k)$-approximation, showing that developing new techniques are necessary to break the $k$ barrier. In addition, we show that our algorithm also provides an $O(d \log^2 d)$-approximate solution where $d$ is the dimension of the space. This is interesting when the dimension is low relative to $k$, specifically $d\in o(\frac {\log k}{\log \log k})$. This improves over the result of Laber and Murtinho~\cite{laber2021explainable} that provides an $O(d \log k)$-approximation algorithm for $k$-medians, since $\min(\log k \log \log k, d \log^2 d)$ is always much smaller than $d \log k$. Note that our result implies a constant-factor approximation algorithm for explainable $k$-medians in constant-dimensional spaces.

Next, we show that our approximation factors for explainable $k$-medians are tight up to a $\log \log k$ factor and a $\log^2 d$ factor, respectively. Specifically, we show that for $d = O(\log k)$, there is no $o(\log k)$-approximation explainable clustering algorithm, which implies an $\Omega(\min(d, \log k))$-approximation lower bound. Previously, there was a known $\Omega(\log k)$-approximation lower bound where $d = \text{poly}(k)$, implying an $\Omega(\min(\log d, \log k))$-approximation lower bound \cite{dasgupta2020explainable}.

Next, we study explainable $k$-means clustering and provide an $O(k \log k)$-approximation algorithm for this problem. We show that this is tight up to a $\log k$ factor by presenting an $\Omega (k)$-approximation hardness result. Our results improve over the $O(k^2)$-approximation algorithm, and $\Omega(\log k)$-hardness result of Dasgupta et al.~\cite{dasgupta2020explainable}. Our results also improve over the $O(d k \log k)$-approximation algorithm of Laber and Murtinho~\cite{laber2021explainable}. 
Finally, as a side result, we provide a $3$-approximation algorithm for explainable $2$-means in arbitrary dimension, which is known to be tight~\cite{dasgupta2020explainable}.
We summarize our results in Table \ref{table:main}.

\begin{table}[htp]
\centering
\begin{tabular}{|c|c|c|c|} 
 \hline
 \textbf{Problem} & \textbf{Alg/LB} & \textbf{Prior Work} & \textbf{Our Work} \\ [0.5ex] 
 \hline\hline
  & \multirow{2}{*}{Algorithm} & $O(k)$ \cite{dasgupta2020explainable} & $O(\log k \log \log k)$ \\ \cline{3-4}
  $k$-medians & & $O(d \log k)$ \cite{laber2021explainable} & $O(d \log^2 d)$ \\  \cline{2-4}
  & Lower Bound & $\Omega(\min(\log d, \log k))$ \cite{dasgupta2020explainable} & $\Omega(\min(d, \log k))$ \\  \hline \hline
  & \multirow{2}{*}{Algorithm} & $O(k^2)$ \cite{dasgupta2020explainable} & \multirow{2}{*}{$O(k \log k)$} \\ \cline{3-3}
  $k$-means & & $O(d k \log k)$ \cite{laber2021explainable} & \\ \cline{2-4}
  & Lower Bound & $\Omega(\min(\log d, \log k))$ \cite{dasgupta2020explainable} & $\Omega(k)$ for $d \ge \Omega(\log k)$ \\ \hline
\end{tabular}
\caption{Summary of approximation algorithms and lower bounds, for both $k$-medians and $k$-means in $d$-dimensional space $\BR^d$. We include both our results and prior results. The approximation ratios are with respect to the optimal \emph{non-explainable} clustering.}
\label{table:main}
\end{table}

We note that, if provided with an $O(1)$-approximate solution to the classical $k$-medians (resp., $k$-means), we provide randomized procedures for explainable $k$-medians (resp., $k$-means) with the above approximation guarantees, that run in only $O(k d \log^2 k)$ (resp., $O(k^2 d)$) time. The runtimes are independent of the total number of points $n,$ and are linear in the dimension $d$. In addition, we provide a deterministic explainable $k$-means algorithm that runs in $O(k d \cdot n \log n)$ time.

Finally, it is worth noting that for both $k$-medians and $k$-means, our randomized algorithms only use an approximate (not necessarily explainable) solution to $k$-means or $k$-medians to construct an explainable clustering and ignore the rest of the data points. Therefore, our algorithms can be combined with a coreset construction, or run on top of another sublinear algorithm for the classic version of the problem and provide an explainable clustering in the same setting. 

\subsection{Other Related Work} 

Explainable $k$-means and $k$-medians clustering have also been studied in practice. Frost et al.~\cite{frost2020exkmc} and Laber and Murtinho~\cite{laber2021explainable} provided practical algorithms for explainable clustering evaluated on real datasets. Other results has also been developed for creating interpretable clustering models or clustering models based on decision trees \cite{bentley75, breiman84, fraiman13, liu05, loh11}.

Due to their applications, the classical $k$-means and $k$-median problems have been studied extensively from both theoretical and practical perspectives with many approximation algorithms and heuristics~\cite{lloyd1982least,arthur2007k,  byrka2014improved, kanungo2004local, jain2003greedy, li20111}. In terms of their computational complexity, these problems are hard to approximate within a factor better than 1.1 in high-dimensional Euclidean spaces and admits approximation schemes in low-dimension~\cite{AroraRR98,KolliopoulosR07,Cohen-Addad18}. On the other hand, they admit constant-factor approximation algorithms for 
high-dimensional Euclidean spaces, better than for general metric spaces~\cite{Cohen-AddadKM19}. 
Due to hardness results, constant-factor approximation factors are not achievable for the explainable clustering formulation.

There are several near-linear time algorithms for the classical $k$-means and $k$-medians\cite{kolliopoulos2007nearly,voevodski2021large,meyerson2004k,cohen2019near,mettu2004optimal,cohen2020fast}.
 In low dimensional Euclidean space Kolliopoulos and Rao provide an approximation scheme for $k$-median in near linear 
$O(f(\epsilon,d)n\log^{d+6})$ time~\cite{kolliopoulos2007nearly}. Recently, Cohen-Addad et al. improved this result and provide a $\tilde{O}(f(\epsilon,d)n)$ time algorithms that w.h.p., give a $1+\eps$ approximation solutions to $k$-median and $k$-means in spaces of doubling dimension $d$~\cite{cohen2019near}.
Mettu and Plaxton provide a randomized algorithm that w.h.p. returns a $O(1)$-approximate solution to $k$-median in time $O(nk)$, when the ratios of the distances do not exceed $2^{O(n/log(n/k))}$~\cite{mettu2004optimal}.

\paragraph{Independent Work.} We note that there have been closely related independent works due to Makarychev and Shan \cite{makarychev2021explainable}, Gamlath, Jia, Polak, and Svensson \cite{gamlath2021explainable}, and Charikar and Hu \cite{charikar2021explainable}. The paper \cite{makarychev2021explainable} provides an $O(\log k \log \log k)$-approximation for $k$-medians, matching ours, and an $O(k \log k \log \log k)$-approximation for $k$-means, an $O(\log \log k)$ factor worse than ours. They also provide guarantees for the related $k$-medoids problem (also known as $k$-medians with $\ell_2$ objective). The paper \cite{gamlath2021explainable} provides an $O(\log^2 k)$-approximation for $k$-medians, slightly under a quadratic factor worse than ours, and an $O(k \log^2 k)$-approximation for $k$-means, an $O(\log k)$-factor worse than ours. They also extend their guarantees to general $\ell_p^p$-objectives. Finally, the paper \cite{charikar2021explainable} looks at $k$-means in low dimensions, and proves an $O(k^{1-2/d} \cdot \text{poly}(d, \log k))$-approximation, which improves over our results for $k$-means if and only if $d \ll \frac{\log k}{\log \log k}$. We are the only paper of these to provide a $o(\log k)$-approximation guarantee for explainable $k$-medians in low dimensions.

\subsection{Preliminaries and Notation} \label{sec:notation}

We let $\mathcal{X} \subset \BR^d$ be a set of $n$ data points, which we wish to cluster. A clustering algorithm partitions $\mathcal{X}$ into $k$ clusters $\mathcal{X}_1, \dots, \mathcal{X}_k$ and assigns a center $\bo{\mu_i} \in \BR^d$ to each cluster $\mathcal{X}_i$. The goal of $k$-medians clustering is to choose the partitioning and centers to minimize $\sum_{i = 1}^{k} \sum_{x \in \mathcal{X}_i} \|x - \bo{\mu_i}\|_1$. The goal of $k$-means clustering is to choose the partitioning and centers to minimize $\sum_{i = 1}^{k} \sum_{x \in \mathcal{X}_i} \|x - \bo{\mu_i}\|_2^2$. In \textbf{explainable clustering}, the partition $\mathcal{X}_1, \dots, \mathcal{X}_k$ must be determined by a decision tree with $k$-leaves, where each decision, or split, is determined by a threshold in a single direction.

In all of our algorithms (both $k$-medians and $k$-means), we start by running a standard (non-explainable) clustering algorithm, which obtains a set of cluster centers $\bo{\mu_1}, \dots, \bo{\mu_k}$. For any $1 \le r \le d$ and any point $x$, we let $x_r$ be the $r$th coordinate of $x$. We also let $\mu_{i, r}$ be the $r$th coordinate of $\bo{\mu_i},$ and let $\mathcal{M}$ be the set of cluster centers $\{\bo{\mu_1}, \dots, \bo{\mu_k}\}$. Also, for any point $x \in \mathcal{X},$ we let $c(x)$ be its closest center in $\mathcal{M}$ (with respect to $\ell_1$ for $k$-medians and $\ell_2$ for $k$-means).
Our algorithms will use $\mathcal{M}$ to produce a decision tree, which we call $T$. Each \textbf{node} $u \in T$, except the root node $u_0$, stores a threshold in a single direction (either $\{x_r < t\}$ or $\{x_r \ge t\}$), representing the decision tree's instruction for when you may traverse from $u$'s parent to $u$. 

Each $u$ has some set of cluster centers contained in $u$, which we call $\mathcal{M}(u)$ -- in our algorithms, each leaf node will have precisely one cluster center.
We let $B(u)$ be the ``box'' determined by the decisions made when splitting (note that some of the dimensions of $B(u)$ may be infinite). So, $\mathcal{M}(u) = \mathcal{M} \cap B(u)$. In a slight abuse of notation, we define $|u| := |\mathcal{M}(u)|,$ i.e., $|u|$ is the number of cluster centers sent to the node $u$. We let $B'(u) \subset B(u)$ be the smallest axis-parallel box that contains $\mathcal{M}(u)$. In other words, $B'(u) = [a_1(u), b_1(u)] \times \cdots \times [a_d(u), b_d(u)],$ where $a_r(u) = \min_{\bo{\mu_i} \in \mathcal{M}(u)} \mu_{i, r}$ and $b_r(u) = \max_{\bo{\mu_i} \in \mathcal{M}(u)} \mu_{i, r}.$ Let $R_r(u) = b_r(u)-a_r(u)$ be the $r$th dimension of the box $B'(u).$ We also let $\mathcal{X}(u) = \mathcal{X} \cap B(u),$ i.e., $\mathcal{X}(u)$ is the set of points in the main pointset $\mathcal{X}$ that are sent to node $u$. Finally, for any point $x \in \mathcal{X},$ we define $s(x)$ as its assigned cluster by the tree $T$. In other words, if $x \in B(u),$ where $u$ is a leaf node, then $s(x)$ is the unique cluster center in $\mathcal{M}(u)$. Our algorithms, with probability $1$, will never create a node $u$ with any point $x$ on the boundary of $B(u)$, so we do not have to worry about points being assigned to multiple leaves.

We also note a few notational conventions. First, we use $\log$ to denote the natural log, unless a base is specified. For any positive integer $n$, we use $[n]$ to denote the set $\{1, 2, \dots, n\}$. We use the inequality $\lesssim$ to mean $a \lesssim b$ if there is some absolute constant $C > 0$ such that $a \le C \cdot b.$

\subsection{Our Techniques}

The methods of both Dasgupta et al.~\cite{dasgupta2020explainable} and Laber and Murtinho~\cite{laber2021explainable} follow a similar strategy. For any node $u$ in the decision tree, if we split $u$ to form two nodes $v, w$, this splitting incurs some cost caused by points in $\mathcal{X}$ that no longer are with their assigned cluster. Dasgupta et al.~\cite{dasgupta2020explainable} shows that each split can be formed in a way that the extra cost of all splits at depth $h$ in the tree does not exceed the total clustering cost. While the depth of the tree is $O(\log k)$ in the best case, the tree they construct could have depth up to $k$, which gives them an $O(k)$-approximation for $k$-medians. (The same technique gives them an $O(k^2)$-approximation for $k$-means.) Laber and Murtinho~\cite{laber2021explainable} instead roughly shows that one can perform the splits so that the cost incurred in each dimension does not significantly exceed the total $k$-medians (or $k$-means) cost.

Unlike the previous work, all of our algorithms ($k$-medians, $k$-means, and $2$-means) either enjoy randomness or are based on a probabilistic analysis. 
Ideally we wish to create an explainable clustering algorithm that maps each $x_i$ to a center $\bo{\mu_s}$ such that $\|x_i-\bo{\mu_s}\|$ is not much larger than $\min_j \|x_i-\bo{\mu_j}\|$, i.e., we map every data point to an approximately optimal cluster center. However this is not possible simultaneously for all points. To compensate for this, we analyze a randomized procedure that upper bounds the \emph{expectation} of $\|x_i-\bo{\mu_s}\|,$ where $\bo{\mu_s}$ is the assigned cluster. Overall, we deviate significantly from \cite{dasgupta2020explainable, laber2021explainable} by comparing the expected cost of each \emph{point} to optimal, as opposed to comparing the (deterministic) cost of each \emph{split} to the optimal clustering cost.

In the case of $k$-medians, the algorithm is fairly simple. Roughly speaking, we iteratively select \emph{uniformly at random} lines until they separate the $k$ centers $\bo{\mu_1}, \dots, \bo{\mu_k}$. In the worst case this procedure is horribly slow; however, it can be sped up with some modifications. For any point $x$ with closest center $c(x)$, we bound the probability that $x$ is assigned to a cluster $s(x)$ with $\|x-s(x)\|_1 \ge T \cdot \|x-c(x)\|_1$, for any integer $T \ge 1$. Note that for $x$ to be assigned to $s(x)$, the random lines must have split $s(x)$ from $x$ before splitting $c(x)$ from $x$. It is easy to show this probability is $O(1/T)$, so a naive union-bound over all $s(x)$ allows us to bound our desired probability by $O(k/T)$.
To improve upon this, we instead note that $x$ must also split not only from $c(x)$, but also from all $T$-approximate (or better) clusters $c'(x)$ before being split from some $s(x)$. In addition, note that the number of lines needed until we finally split $x$ from $s(x)$ is a Geometric random variable, so it exceeds its expectation by a multiplicative factor of $\log k$ with only $1/k$ probability, meaning with high probability, no faraway cluster $s(x)$ takes too long to get split.
By considering the different orderings in which random lines split $x$ from $c(x)$ and all $c'(x)$ with $\|x-c'(x)\|_1 \le T \|x-c(x)\|_1$, we provide a complicated upper bound on this probability that depends on the locations of all cluster centers. Finally, by integrating over $T$ we obtain a telescoping sum which provides an $O(\log k \log \log k)$-approximation for the cost of each \emph{point} $x$.


In the case of $k$-means, we start by following the deterministic approach of Dasgupta et al.~\cite{dasgupta2020explainable}, but we add a randomized twist that improves the analysis. At a high level, \cite{dasgupta2020explainable} shows that at each step, it is possible to choose a splitting line with sufficiently few points being mapped to the wrong cluster center. 
Unfortunately, as mentioned previously, this tree can be very lopsided and have depth $k$, which can blow up the approximation factor. To resolve this issue, we create a distribution over choosing separating lines that balances the errors of the splitting line with the lopsidedness of the points at each step. This distribution is somewhat based on the uniformly random procedure in the $k$-medians case, but modified to deal with the issues of squared costs. 
This combination of creating a non-trivial random distribution with balancing errors and lopsidedness reduces the total clustering cost significantly. 
However, we note that for this $k$-means algorithm, the randomization is primarily helpful for the \emph{analysis}, so our algorithm can either remain randomized (which allows for sublinear guarantees, as in the $k$-medians case) or be made deterministic.

\section{Algorithm for Explainable $k$-medians Clustering} \label{sec:kmedians}

In this section, we provide both an $O(\log k \log \log k)$ and an $O(d \log^2 d)$-approximation algorithm (in expectation) for explainable $k$-medians clustering. We start by describing and analyzing a simplified algorithm that is accurate but can be very slow. We then show how to modify the algorithm to be highly efficient, and prove that the approximation guarantees still hold.

\subsection{Simplified algorithm} \label{subsec:kmedians_simplified}

Our simplified algorithm works as follows. First, we run some standard $k$-medians algorithm that provides an $O(1)$-approximation with $k$ centers $\bo{\mu_1}, \bo{\mu_2}, \dots, \bo{\mu_k} \in \BR^d.$ 
We suppose that all of the centers $\bo{\mu_1}, \dots, \bo{\mu_k}$, as well as all of the points, are contained in $[-B, B]^d$ for some large $B$. Our approximation factor does not depend on $B$, so $B$ may be arbitrarily large.
We first consider the following simplified procedure. At each step $i$, we pick a random direction $r \in [d]$, as well as a random point $z \in [-B, B]$. We choose the separating line $\ell_i = \{x_r = z\}$. However, we only use this line to separate a (currently leaf) node $u$ if this line actually separates some of the cluster centers in that node (or equivalently, splits the cell $B'(u)$). Note that that often, the line may not be used at all. Assuming the line is actually used, each leaf node for which the line is used is split into $2$ child nodes. We repeat this process until there are $k$ leaf nodes, each with exactly one center in it.

For any pair of cluster centers $\bo{\mu_i}, \bo{\mu_j},$ note that the probability of a randomly selected line $\ell$ separating these two centers is precisely $\|\bo{\mu_i}-\bo{\mu_j}\|_1/(B \cdot d).$ Therefore, in expectation we should expect about $(B \cdot d)/\|\bo{\mu_i}-\bo{\mu_j}\|_1$ random lines to be chosen before $\bo{\mu_i}$ and $\bo{\mu_j}$ are separated. 

Fix a point $x \in \mathcal{X}$, and suppose that the closest center to $x$ is some $\bo{\mu_r}$. Our main result, which will allow us to obtain both an $O(\log k \log \log k)$-approximation and an $O(d \log^2 d)$-approximation algorithm, is the following.

\begin{theorem} \label{thm:kmedian_main}
    Fix any point $x \in \BR^d$ and any $k$ clusters $\bo{\mu_1}, \dots, \bo{\mu_k} \in [-B, B]^d,$ and define $c := \arg\min_{1 \le i \le k} \|x-\bo{\mu_i}\|_1.$ Suppose that our randomized explainable clustering procedure assigns $x$ to cluster $\bo{\mu_s}$. Then,
$$\BE \left[\|x-\bo{\mu_s}\|_1\right] \le O(\log k \cdot \log \log k) \cdot \|x-\bo{\mu_c}\|_1.$$
    In addition,
$$\BE \left[\|x-\bo{\mu_s}\|_1\right] \le O(d \cdot \log^2 d) \cdot \|x-\bo{\mu_c}\|_1.$$
\end{theorem}

To see why this implies our final result, consider any dataset $\mathcal{X}$ and any $k$ clusters $\bo{\mu_1}, \dots, \bo{\mu_k}$ that form an $O(1)$-approximation for $k$-medians clustering. If we define $c(x)$ to be the closest cluster center to $x$ (i.e., the ``true'' center) and $s(x)$ to be the assigned cluster center to $x$ by the explainable algorithm, then by Theorem \ref{thm:kmedian_main} and Linearity of Expectation,
\[\BE\left[\sum_{x \in \mathcal{X}} \|x-s(x)\|_1\right] \le O(\log k \log \log k) \cdot \sum_{x \in \mathcal{X}} \|x-c(x)\|_1 = O(\log k \log \log k) \cdot \text{OPT},\]
where $\text{OPT}$ is the optimal clustering cost.
Likewise,
\[\BE\left[\sum_{x \in \mathcal{X}} \|x-s(x)\|_1\right] \le O(d \cdot \log^2 d) \cdot \sum_{x \in \mathcal{X}} \|x-c(x)\|_1 = O(d \cdot \log^2 d) \cdot \text{OPT}.\]
Hence, we obtain both an $O(\log k \log \log k)$ and an $O(d \log^2 d)$ approximation guarantee.

\subsection{Proof of Theorem \ref{thm:kmedian_main}} \label{subsec:kmedians_proof}

Assume WLOG that $\bo{\mu_1}, \bo{\mu_2}, \dots, \bo{\mu_k}$ are sorted so that $\|x-\bo{\mu_1}\|_1 \le \|x-\bo{\mu_2}\|_1 \le \cdots \le \|x-\bo{\mu_k}\|_1$ (so we assume that $c = 1$).
In addition, assume WLOG that we scale and shift the dataset so that $\|x-\bo{\mu_1}\|_1 = 1$ and $x$ is at the origin. By redefining $B$ if necessary, we still assume that $x$ and $\bo{\mu_1}, \dots, \bo{\mu_k}$ are all contained in $[-B, B]^d$.
In addition, we partition the set $[k]$ into contiguous subsets $S_0, S_1, S_2, \dots,$ where $i \in S_h$ if $2^{h} \le \|x-\bo{\mu_i}\|_1 < 2^{h+1}.$ Note that $1 \in S_0$.
Finally, for each integer $H$, we define $P(H)$ as the largest index in $\bigcup_{h \le H} S_h$. Note that even if $S_i$ is empty, $P(H)$ is well defined since $1 \in S_0$, so therefore, $1 \le P(0) \le P(1) \le \dots.$

For any integer $H \ge 2$, we will bound the probability that our procedure assigns $x$ to some cluster $\bo{\mu_s}$ for some $s \in S_H$ in two ways: first in terms of the number of clusters $k$, and second in terms of the dimension $d$.
Note that if $x$ is assigned to $\bo{\mu_s}$, then for all $1 \le p < s$, the first time that we randomly chose a line $\ell$ that separated $\bo{\mu_p}$ and $\bo{\mu_s}$, $x$ was on the same side as of the line as $\bo{\mu_s}$. This is because this line will be used to separate $\bo{\mu_p}$ and $\bo{\mu_s}$ in the explainable clustering procedure, as it is the first sampled line that separates them, and if $x$ were on the same side as $\bo{\mu_p}$, it could not be assigned to $\bo{\mu_s}$.
So, if $x$ is assigned to $\bo{\mu_s}$ for some $s \in S_H$, there are two options:
\begin{enumerate}
    \item Let $p = P(H-2).$ Then, there exists $s \in S_H$ such that the first sampled line that splits $\bo{\mu_1}$ from $\bo{\mu_s}$ splits $\bo{\mu_1}, \dots, \bo{\mu_p}$ from $x$ and $\bo{\mu_s}$.
    \item There exists $1 \le p < P(H-2)$ and $s \in S_H$ such that the first sampled line that splits $\bo{\mu_1}$ from $\bo{\mu_s}$ splits $\bo{\mu_1}, \dots, \bo{\mu_p}$ from $x$ and $\bo{\mu_{p+1}}$. In addition, the first sampled line that splits $\bo{\mu_{p+1}}$ from $\bo{\mu_s}$ splits $\bo{\mu_{p+1}}$ from $x$. 
\end{enumerate}

For each $H \ge 2$ and $p \le P(H-2)$, we let $\mathcal{A}(p, H)$ be the event that the corresponding option occurs (option $1$ for $p = P(H-2)$, option $2$ for $p < P(H-2)$). By the union bound over $p$, the probability that $x$ is assigned to $\bo{\mu_s}$ for some $s \in S_H$ is at most $\sum_{p \le P(H-2)} \BP(\mathcal{A}(p, H)).$ Therefore, since $\|x-\bo{\mu_s}\|_1$ is $O(2^H)$ if $s \in s_H$, we have that for any $V \ge 1,$
\begin{equation} \label{eq:main_bound}
    \BE[\|x-\bo{\mu_s}\|_1] \le O\left(V + \sum_{H: 2^H \ge V} 2^H \cdot \sum_{p \le P(H-2)} \BP(\mathcal{A}(p, H))\right).
\end{equation}
The additional $O(V)$ term comes from the fact that with some probability, we pick a cluster $\bo{\mu_s}$ with $s \in S_H$ for some $H$ satisfying $2^H \le V$, in which case $\|x-\bo{\mu_s}\|_1 \le 2V.$

Before we get to bounding $\BP(\mathcal{A}(p, H))$, 
we make some definitions. For two values $a, b \in \BR,$ we define $a \land b = \min(a, b)$ if $a, b \ge 0$, $\max(a, b)$ if $a, b \le 0$, and $0$ if $a < 0 < b$ or $b < 0 < a$. Note that this operation is associative (and commutative), so we can define $a_1 \land a_2 \land \cdots \land a_n$ in the natural fashion. In general, for points $x_1, \dots, x_n \in \BR^d$, we define $x_1 \land x_2 \land \cdots \land x_n$ coordinate-wise. Note that a line separates $x_1, \dots, x_n$ from the origin $x = \textbf{0}$ if and only if the line separates $x_1 \land \cdots \land x_n$ from $x$.
Next, for each $1 \le p \le k,$ we define $c_p = \|\bo{\mu_p}\|_1 = \|x-\bo{\mu_p}\|_1$ (recall that we assumed $x$ was the origin). Note that the probability of a randomly sampled line splitting $x$ from $\bo{\mu_p}$ is $c_p/(B d)$. In addition, for each $1 \le p \le k,$ define $\alpha_p = \|\bo{\mu_1} \land \cdots \land \bo{\mu_p}\|_1$, which equals $Bd$ times the probability that a randomly sampled line splits $x$ from $\bo{\mu_1}, \dots, \bo{\mu_p}$. Finally, for any $J \ge 0$, define $\beta_J = \alpha_{P(J)}$, or equivalently, $\beta_J$ equals $Bd$ times the probability that a randomly sampled line splits $x$ from $\bo{\mu_i}$ for all $i \in \bigcup_{h \le J} S_h$. We note that $1 = \alpha_1 \ge \alpha_2 \ge \ldots$ and $1 \ge \beta_0 \ge \beta_1 \ge \ldots$. For convenience, we define $\beta_{-1} := \alpha_1 = 1.$


We also note the following simple proposition, which will be useful in bounding probabilities.

\begin{proposition} \label{prop:basic_ineq}
    Let $N \ge 1$ and $0 < \eps < 1$. Then, $\sum_{t = 1}^{\infty} \min\left(N \cdot (1-\eps)^{t-1}, 1\right) \le (\log(N) + 1)/\eps$.
\end{proposition}

\begin{proof}
    For $t \le \log(N)/\eps$, we can use bound that $\min(N \cdot (1-\eps)^{t-1}, 1) \le 1$. Else, we write $t = t' + 1 + (\log(N)/\eps),$ where $t' \ge 0$, and $N \cdot (1-\eps)^{t-1} = (1-\eps)^{t'} \cdot N \cdot (1-\eps)^{\log (N)/\eps} \le (1-\eps)^{t'} \cdot N \cdot e^{-\eps \cdot \log(N)/\eps} = (1-\eps)^{t'}$. Therefore,
\[\sum_{t = 1}^{\infty} \min\left(N \cdot (1-\eps)^{t-1}, 1\right) \le \sum_{t = 1}^{\log(N)/\eps} 1 + \sum_{t' = 0}^{\infty} (1-\eps)^{t'} = \frac{\log(N)}{\eps} + \frac{1}{\eps} = \frac{\log(N) + 1}{\eps}. \qedhere\]
\end{proof}

We now provide an upper bound on $\BP(P(H-2), H).$

\begin{lemma} \label{lem:APH_bound_1}
    Let $p = P(H-2)$. Then, $\BP(\mathcal{A}(p, H)) \le C \cdot \log(k) \cdot \frac{\beta_{H-2}-\beta_H}{2^H}$ for some absolute constant $C$. In addition, if $2^H \ge 2d,$ then $\BP(\mathcal{A}(p, H)) \le C \cdot d \log(d) \cdot \frac{\beta_{H-2}-\beta_H}{2^H}.$
\end{lemma}

\begin{proof}
    Define $\mathcal{E}_0$ to be the event, and $\gamma_0$ to be the associated probability, that a randomly sampled line splits $\bo{\mu_1}, \dots, \bo{\mu_p}$ from $x$ and $\bo{\mu_s}$ for some $s \in S_H$. Note that $\gamma_0 \le (\beta_{H-2}-\beta_H)/(Bd)$, because the probability that a randomly sampled line splits $\bo{\mu_1}, \dots, \bo{\mu_p}$ from $x$ is at most $\beta_{H-2}/(Bd)$, but you have to subtract a quantity that a randomly sampled line splits $\bo{\mu_1}, \dots, \bo{\mu_p}$ as well as $\bo{\mu_s}$ for all $s \in S_H$ from $x$, which is at least $\beta_H/(Bd)$.
    Next, for each $s \in S_H$, define $\mathcal{E}_s$ to be the event, and $\gamma_s$ to be the associated probability, that a randomly sampled line splits $\bo{\mu_1}$ and $x$ from $\bo{\mu_s}$. Note that $\gamma_s \ge c_s - c_1$, because $\mathcal{E}_s$ occurs as long as we split $\bo{\mu_s}$ from $x$, but don't split $\bo{\mu_1}$ from $x$. Therefore, $\gamma_s \ge (2^H-1)/(B \cdot d) \ge 2^H/(2 \cdot B \cdot d).$ Let $\gamma' = \min_{s \in S_H} \gamma_s \ge 2^H/(2 \cdot B \cdot d).$
    
    If $\mathcal{A}(p, H)$ occurs, then $\mathcal{E}_0$ must occur before $\mathcal{E}_s$ for some $s \in S_H$. The probability that $\mathcal{E}_0$ occurs at some time step $t$ but for some $s \in S_H$, $\mathcal{E}_s$ did not occur for any time step before $t$, is at most $\gamma_0 \cdot \sum_{s \in S_H} (1-\gamma_s)^{t-1} \le \gamma_0 \cdot k \cdot (1-\gamma')^{t-1}$, by the union bound and since what happens at each time step is independent. In addition, we can also bound this probability by just $\gamma_0$ by ignoring the event that $\mathcal{E}_s$ did not occur for any time step before $t$. Therefore, for $p = P(H-2)$,
\begin{align}
    \BP(\mathcal{A}(p, H)) &\le \sum_{t = 1}^{\infty} \gamma_0 \cdot \min\left(k \cdot (1-\gamma')^{t-1}, 1\right) \nonumber \\ 
    &\le (\log(k) + 1) \cdot \frac{\gamma_0}{\gamma'} \nonumber \\
    &\lesssim \log(k) \cdot \frac{\beta_{H-2}-\beta_H}{2^H}, \label{eq:APH_1}
\end{align}
    where the second inequality follows from Proposition \ref{prop:basic_ineq}.

Next, suppose $2^H \ge 2d$. For each dimension $1 \le r \le d$ we define $\mathcal{D}_{r, +}$ as the event that a randomly selected line is of the form $\{x_r = z\}$ for some $z \in [2^H/(2d), 2^H/d]$. Likewise, we define $\mathcal{D}_{r, -}$ as the event that a randomly selected line is of the form $\{x_r = z\}$ for some $z \in [-2^H/d, -2^H/(2d)].$
Note that the probability of each $\mathcal{D}_{r, +}$ and each $\mathcal{D}_{r, -}$ is precisely $2^H/(2d \cdot Bd)$, which we call $\delta'$.
Note that if $\mathcal{A}(p, H)$ occurs, then $\mathcal{E}_0$ must occur before $\mathcal{D}_{r, +}$ or $\mathcal{D}_{r, -}$ for some $r \in [d]$ and some choice of $+/-$. This is because any $\bo{\mu_s}$ must have at least one of its coordinates larger than $2^H/d$ in absolute value, so one of these lines must separate $\bo{\mu_s}$ from $x$. But since $2^H/(2d) > 1 = \|\bo{\mu_1}\|_1,$ $x$ is on the same side as $\bo{\mu_1}$, so if every $\mathcal{D}_{r, +}$ and $\mathcal{D}_{r, -}$ occurs before $\mathcal{E}_0$, then the splitting of $\bo{\mu_1}$ from $\bo{\mu_s}$ will occur before the splitting of $\bo{\mu_1}, \dots, \bo{\mu_p}$ from $x$ and $\bo{\mu_s}$.
The probability that $\mathcal{E}_0$ occurs at some time step $t$ but some choice of $\mathcal{D}_{r, +}$ or $\mathcal{D}_{r, -}$ did not occur for any time step before $t$, is at most $\min(\gamma_0, \gamma_0 \cdot 2d \cdot (1-\delta')^{t-1}).$
Therefore,
\begin{align}
\BP(\mathcal{A}(p, H)) &\le \sum_{t = 1}^{\infty} \gamma_0 \cdot \min(2d \cdot (1-\delta')^{t-1}, 1) \nonumber \\
&\le (\log(2 d) + 1) \cdot \frac{\gamma_0}{\delta'} \nonumber \\
&\lesssim d \log(d) \cdot \frac{\beta_{H-2}-\beta_H}{2^H}. \label{eq:APH_2}
\end{align}
Again, the second inequality follows by Proposition \ref{prop:basic_ineq}.

Combining Equations \eqref{eq:APH_1} and \eqref{eq:APH_2}, the lemma is complete.
\end{proof}

Next, we provide an upper bound on $\BP(\mathcal{A}(p, H))$ for $p < P(H-2)$.

\begin{lemma} \label{lem:APH_bound_2}
    For $p < P(H-2),$ $\BP(\mathcal{A}(p, H)) \le C \cdot (\log k) \cdot \frac{\alpha_p-\alpha_{p+1}}{2^H} \cdot \min\left((\log k) \cdot \frac{c_{p+1}}{2^H}, 1\right).$
\end{lemma}

\begin{proof}
    We redefine $\mathcal{E}_0$ to be the event, and $\gamma_0$ to be the associated probability, that a randomly sampled line splits $\bo{\mu_1}, \dots, \bo{\mu_{p}}$ from $x$ and $\bo{\mu_{p+1}}$. For $p < P(H-2),$ we have $\gamma_0 = (\alpha_p-\alpha_{p+1})/(B \cdot d).$ Next, for each $s \in S_H$, we keep the same definition of $\mathcal{E}_s$ and $\gamma_s$ corresponding to a randomly sampled line splitting $\bo{\mu_1}$ and $x$ from $\bo{\mu_s}$. Recall that $\gamma_s \ge 2^H/(2 \cdot B \cdot d)$, and $\gamma' = \min_{s \in S_H} \gamma_s \ge 2^H/(2 \cdot B \cdot d).$
    
    Next, define $\mathcal{H}_0$ to be the event, and $\eta_0$ to be the associated probability, that a randomly sampled line splits $\bo{\mu_{p+1}}$ from $x$. Clearly, $\eta_0 = c_{p+1}/(B \cdot d)$. Next, for each $s \in S_H,$ define $\mathcal{H}_s$ to be the event, and $\eta_s$ to be the associated probability, that a randomly sampled line splits $\bo{\mu_s}$ from $x$ and $\bo{\mu_{p+1}}$. This probability is at least the probability of a line splitting $\bo{\mu_s}$ from $x$ minus the probability of a line splitting $\bo{\mu_{p+1}}$ from $x$, which is $(c_s-c_{p+1})/(B \cdot d) \ge 2^H/(2 \cdot B \cdot d)$, since $c_s \ge 2^H$ as $s \in S_H$ and $c_{p+1} \le 2^{H-1}$ as $p+1 \le P(H-2)$. So, if we define $\eta' = \min_{s \in S_H} \eta_s$, then $\eta' \ge 2^H/(2 \cdot B \cdot d).$
    
    Now, for $\mathcal{A}(p, H)$ to occur, if $t$ is the first time that $\mathcal{E}_0$ occurs and $u$ is the first time that $\mathcal{H}_0$ occurs, then there must exist some $s$ such that $\mathcal{E}_s$ does not occur before $t$ and $\mathcal{H}_s$ does not occur before $u$. Note that $\mathcal{E}_0$ and $\mathcal{H}_0$ are disjoint events, which means $t \neq u$. If $t < u$ and we write $u = t+t'$, then we must have that $\exists s \in S_H$ such that for all $1 \le i \le t-1$, $\mathcal{E}_s$ doesn't occur. In addition, $\exists s \in S_H$ such that for all $t+1 \le i \le t+t'-1$, $\mathcal{H}_s$ doesn't occur. Therefore, we can bound the sum over $t < u$ of $\mathcal{A}(p, H)$ occurring where $\mathcal{E}_0$ first occurs at time $t$ and $\mathcal{H}_0$ first occurs at time $u$ as
\begin{align}
    &\le \sum_{t = 1}^{\infty} \sum_{t' = 1}^{\infty} \gamma_0 \eta_0 \cdot \min(k \cdot (1-\gamma')^{t-1}, 1) \cdot \min(k \cdot (1-\eta')^{t'-1}, 1) \nonumber \\
    &= \left(\sum_{t = 1}^{\infty} \gamma_0 \cdot \min(k \cdot (1-\gamma')^{t-1}, 1)\right) \cdot \left(\sum_{t' = 1}^{\infty} \eta_0 \cdot \min(k \cdot (1-\eta')^{t'-1}, 1)\right) \label{eq:APH_product} \\
    &\le (\log k + 1) \cdot \frac{\gamma_0}{\gamma'} \cdot (\log k + 1) \cdot \frac{\eta_0}{\eta'} \nonumber \\
    &\lesssim (\log k)^2 \cdot \frac{\gamma_0}{\gamma'} \cdot \frac{\eta_0}{\eta'}. \label{eq:APH_3}
\end{align}
    Likewise, if $u < t$, we write $t = u+u'$ where $u, u'$ range from $1$ to $\infty$, and we obtain the same product as Equation \eqref{eq:APH_product} and upper bound as Equation \eqref{eq:APH_3}. Therefore, we have that
\begin{equation}
    \BP(\mathcal{A}(p, H)) \lesssim (\log k)^2 \cdot \frac{\gamma_0}{\gamma'} \cdot \frac{\eta_0}{\eta'} \lesssim (\log k)^2 \cdot \frac{\alpha_p-\alpha_{p+1}}{2^H} \cdot \frac{c_{p+1}}{2^H}. \label{eq:APH_4}
\end{equation}
    
    Finally, we also can bound $\mathcal{A}(p, H)$ by completely ignoring $\mathcal{H}_0$ and $\mathcal{H}_s$, and just considering $\mathcal{E}_0$ occuring before some $\mathcal{E}_s$. By similar calculations to Lemma \ref{lem:APH_bound_1}, this results in the bound
\begin{equation}
    \mathcal{P}(\mathcal{A}(p, H)) \le \sum_{t = 1}^{\infty} \gamma_0 \cdot \min(k \cdot (1-\gamma')^{t-1}, 1) \lesssim (\log k) \cdot \frac{\gamma_0}{\gamma'} \lesssim (\log k) \cdot \frac{\alpha_p-\alpha_{p+1}}{2^H}. \label{eq:APH_5}
\end{equation}
    By combining Equations \eqref{eq:APH_4} and \eqref{eq:APH_5}, we obtain
\begin{equation*}
\BP(\mathcal{A}(p, H)) \lesssim (\log k) \cdot \frac{\alpha_p-\alpha_{p+1}}{2^H} \cdot \min\left((\log k) \cdot \frac{c_{p+1}}{2^H}, 1\right). \qedhere
\end{equation*}
\end{proof}

Next, we provide an upper bound on $\BP(\mathcal{A}(p, H))$ for $p < P(H-2)$ based on the dimension $d$.

\begin{lemma} \label{lem:APH_bound_3}
    For $p < P(H-2)$ and $2^H \ge 2d,$ $\BP(\mathcal{A}(p, H)) \le C \cdot (d \log d) \cdot \frac{\alpha_p-\alpha_{p+1}}{2^H} \cdot \min\left((d \log d) \cdot \frac{c_{p+1}}{2^H}, 1\right).$
\end{lemma}

\begin{proof}
As in Lemma \ref{lem:APH_bound_2},
we define $\mathcal{E}_0$ as the event, and $\gamma_0 = \frac{\alpha_p-\alpha_{p+1}}{B \cdot d}$ as the associated probability, that a randomly sampled line splits $\bo{\mu_1}, \dots, \bo{\mu_p}$ from $x$ and $\bo{\mu_{p+1}}$. 
Also, as in Lemma \ref{lem:APH_bound_2},
we define $\mathcal{H}_0$ as the event, and $\eta_0 = c_{p+1}/(B \cdot d)$ as the associated probability, that a randomly sampled line splits $\bo{\mu_{p+1}}$ from $x.$
Finally, as in Lemma \ref{lem:APH_bound_1}, let $\mathcal{D}_{r, +}$ be the event that a line is of the form $\{x_r = z\}$ for $z \in [2^H/(2d), 2^H/d]$ and $\mathcal{D}_{r, -}$ be the event that a line is of the form $\{x_r = z\}$ for $z \in [-2^H/d, -2^H/(2d)]$.
We let $\delta' = 2^H/(2d \cdot Bd)$ be each of these event's probabilities.

Consider $t$ as the first time that $\mathcal{E}_0$ occurs, and $u$ as the first time that $\mathcal{H}_0$ occurs. (Recall that $t \neq u$ since these two events are disjoint.)
In order for $\mathcal{A}(p, H)$ to occur, we must have that for all $i < t$, one of the $2d$ intervals $[2^H/2d, 2^H/d]$ or $[-2^H/d, -2^H/2d]$ is never covered, since otherwise, we would have had an earlier splitting of $\bo{\mu_1}$ and $x$ from $\bo{\mu_s}$.
In addition, if $2^H/(2d) \ge c_{p+1}$, then for this must also be true for all $i < u,$ or else we would have also had an earlier splitting of $\bo{\mu_{p+1}}$ and $x$ from $\bo{\mu_s}$.
First, we suppose that $2^H/(2d) \ge c_{p+1}$.
If we just consider the case where $t < u$, this implies there exists an interval that isn't covered by any of the lines $1 \le i \le t-1$ and there also is an interval that isn't covered by any of the lines $t+1 \le i \le u-1$. By writing $u = t + t'$, we can bound the sum over $t < u$ of $\mathcal{A}(p, H)$ occurring where $\mathcal{E}_0$ first occurs at time $t$ and $\mathcal{H}_0$ first occurs at time $u$ as
\begin{align*}
    &\le \sum_{t = 1}^{\infty} \sum_{t' = 1}^{\infty} \gamma_0 \eta_0 \cdot \min(2d \cdot (1-\delta')^{t-1}, 1) \cdot \min(2d \cdot (1-\delta')^{t-1}, 1) \\
    &= \left(\sum_{t = 1}^{\infty} \gamma_0 \cdot \min(2d \cdot (1-\delta')^{t-1}, 1)\right) \cdot \left(\sum_{t' = 1}^{\infty} \eta_0  \cdot \min(2d \cdot (1-\delta')^{t-1}, 1)\right) \\
    &\le (\log (2d)+1)^2 \cdot \frac{\gamma_0}{\delta'} \cdot \frac{\eta_0}{\delta'} \\
    &\lesssim (d \log d)^2 \cdot \frac{\alpha_p-\alpha_{p+1}}{2^H} \cdot \frac{c_{p+1}}{2^H}.
\end{align*}
In addition, by summing over the cases where $t > u$ (for instance writing $t = u+u'$), we get an identical sum, so
\begin{equation} \label{eq:APH_6}
    \BP(\mathcal{A}(p, H)) \lesssim (d \log d)^2 \cdot \frac{\alpha_p-\alpha_{p+1}}{2^H} \cdot \frac{c_{p+1}}{2^H}.
\end{equation}

In addition, by ignoring $\mathcal{H}_0$ and just computing the probability that a line splits $\bo{\mu_1}, \dots, \bo{\mu_p}$ from $x$ and $\bo{\mu_{p+1}}$ before one of the $2d$ intervals, the same argument as in Lemma \ref{lem:APH_bound_1} allows us to bound 
\begin{equation} \label{eq:APH_7}
    \BP(\mathcal{A}(p, H)) \lesssim (d \log d) \cdot \frac{\alpha_p-\alpha_{p+1}}{2^H}.
\end{equation} 
This does not require $2^H/(2d) \ge c_{p+1}$, just that $2^H \ge 2d$. In addition, note that if $2^H/(2d) < c_{p+1}$, then $(d \log d) \cdot \frac{\alpha_p-\alpha_{p+1}}{2^H} \le (d \log d)^2 \cdot \frac{\alpha_p-\alpha_{p+1}}{2^H} \cdot \frac{c_{p+1}}{2^H}$ anyway.
So, by combining Equations \eqref{eq:APH_6} and \eqref{eq:APH_7}, we obtain 
\[\BP(\mathcal{A}(p, H)) \lesssim (d \log d) \cdot \frac{\alpha_p-\alpha_{p+1}}{2^H} \cdot \min\left(d \log d \cdot \frac{c_{p+1}}{2^H}, 1\right). \qedhere\]
\end{proof}

We now return to the proof of Theorem \ref{thm:kmedian_main}. Define $f = \min(d \log d, \log k)$.
Note that for any $1 \le J < H,$ $\sum_{p+1 \in S_J} (\alpha_p-\alpha_{p+1}) = \beta_{J-1}-\beta_J,$ and if $p+1 \in S_J,$ then $c_{p+1} \le 2 \cdot 2^J.$ This is also true for $J = 0$, since $\sum_{p \ge 1: p+1 \in S_J} (\alpha_p-\alpha_{p+1}) = \beta_{-1}-\beta_0$ by our definition of $\beta_{-1} := \alpha_1$, and $c_{p+1} \le 2 = 2 \cdot 2^J$. Therefore, by adding either Lemma \ref{lem:APH_bound_2} or Lemma \ref{lem:APH_bound_3} over all $p < P(H-2)$ and splitting based on which set $S_J$ contains $p+1$, we get that for any $H \ge 2$ (if $f = \log k$) or for any $2^H \ge 2d$ (if $f = d \log d$),
\begin{align*}
\sum_{p < P(H-2)} \BP(\mathcal{A}(p, H)) &\le C \cdot f \cdot \sum_{J = 0}^{H-2} \sum_{p+1 \in S_J} \frac{\alpha_p-\alpha_{p+1}}{2^H} \cdot \min\left(\frac{f \cdot c_{p+1}}{2^H}, 1\right) \\
&\le 2C \cdot f \cdot \sum_{J = 0}^{H-2} \frac{\beta_{J-1}-\beta_J}{2^H} \cdot \min\left(f \cdot \frac{2^J}{2^H}, 1\right)
\end{align*}

Adding the term for $\BP(\mathcal{A}(P(H-2), H))$ based on Lemma \ref{lem:APH_bound_1}, we obtain
\[\sum_{p \le P(H-2)} \BP(\mathcal{A}(p, H)) \le 2C \cdot f \cdot \sum_{J = 0}^{H} \frac{\beta_{J-1}-\beta_J}{2^H} \cdot \min\left(f \cdot \frac{2^J}{2^H}, 1\right).\]
Therefore, 
\begin{align}
    \sum_{\substack{H \ge 2 \text{ if } f = \log k \\ 2^H \ge 2d \text{ if } f = d \log d}} 2^H \cdot \sum_{p \le P(H)} \BP(\mathcal{A}(p, H)) &\le 2 C \cdot f \cdot \sum_{H \ge 2} 2^H \cdot \sum_{J = 0}^{H} \frac{\beta_{J-1}-\beta_J}{2^H} \cdot \min\left(f \cdot \frac{2^J}{2^H}, 1\right) \nonumber \\
    &= 2 C \cdot f \cdot \sum_{H \ge 2} \sum_{J = 0}^{H} (\beta_{J-1}-\beta_J) \cdot \min\left(f \cdot \frac{2^J}{2^H}, 1\right) \nonumber \\
    &\le 2C \cdot f \cdot \sum_{J \ge 0} (\beta_{J-1}-\beta_J) \cdot \sum_{H \ge J} \min\left(f \cdot \frac{2^J}{2^H}, 1\right) \nonumber \\
    &\le 2C \cdot f \cdot \sum_{J \ge 0} (\beta_{J-1}-\beta_J) \cdot (\log f + 2) \nonumber \\ 
    &\lesssim f \cdot \log f \cdot \beta_{-1} = f \cdot \log f. \label{eq:main_ineq}
\end{align}

To finish, if $f = \log k$, then by Equations \eqref{eq:main_bound} (where we set $V = 4$ so that $H \ge 2$) and \eqref{eq:main_ineq},
$$\BE[\|x-\bo{\mu_s}\|_1] \le O\left(4 + \sum_{H \ge 2} 2^H \cdot \sum_{p \le P(H)} \BP(\mathcal{A}(p, H))\right) \le O(\log k \cdot \log \log k),$$
assuming that $C$ is a constant.
If $f = d \log d$, then by Equations \eqref{eq:main_bound} (where we set $V = 2d$ so that $2^H \ge 2d$) and \eqref{eq:main_ineq},
$$\BE[\|x-\bo{\mu_s}\|_1] \le O\left(2d + \sum_{H: 2^H \ge 2d} 2^H \cdot \sum_{p \le P(H)} \BP(\mathcal{A}(p, H))\right) \le O(d \cdot \log^2 d).$$

This concludes the proof of Theorem \ref{thm:kmedian_main}.

\subsection{Faster algorithm} \label{subsec:kmedians_faster}

\begin{figure}
\centering
\begin{minipage}{0.47\textwidth}
\begin{algorithm}[H]
    \caption{Main procedure for explainable $k$-medians}
    \begin{algorithmic}[1] 
        \Procedure{K-medians}{$u$}
            \State Use standard $k$-medians algorithm to find centers $\bo{\mu_1}, \dots, \bo{\mu_k}$
            \State \textbf{Create} tree $T$ with single node $u_0 \leftarrow \emptyset$ with $\mathcal{M}(u_0) = \{\bo{\mu_1}, \dots, \bo{\mu_k}\}$
            \While{$\exists$ leaf $u \in T$ with $|\mathcal{M}(u)| \ge 2$}
                \State \Call{MedianSplit}{u}
            \EndWhile
            \State \textbf{Return} $T$
        \EndProcedure
    \end{algorithmic}
\end{algorithm}
\end{minipage}
\hfill
\begin{minipage}{0.51\textwidth}
\begin{algorithm}[H]
    \caption{Splitting procedure of a node $u$}
    \begin{algorithmic}[1] 
        \Procedure{MedianSplit}{$u$}
            \For{$r = 1$ to $d$}
                \State $a_r = \mathop{\min}\limits_{\bo{\mu_i} \in \mathcal{M}(u)} \mu_{i, r}$
                \State $b_r = \mathop{\max}\limits_{\bo{\mu_i} \in \mathcal{M}(u)} \mu_{i, r}$
                \State $R_r = b_r - a_r$
            \EndFor
            \State \textbf{Sample} $r \in [d]$ with probability $\frac{R_r}{R_1+\cdots+R_d}$
            \State \textbf{Sample} $z \sim \text{Unif}[a_r, b_r]$
            \State \textbf{Add} left child $\mathcal{L}(u) \leftarrow \{x_r < z\}$ to $u$
            \State $\mathcal{M}(\mathcal{L}(u)) = \mathcal{M}(u) \cap \{\bo{\mu_i}: \mu_{i, r} < z\}$
            \State \textbf{Add} right child $\mathcal{R}(u) \leftarrow \{x_r \ge z\}$ to $u$
            \State $\mathcal{M}(\mathcal{R}(u)) = \mathcal{M}(u) \cap \{\bo{\mu_i}: \mu_{i, r} \ge z\}$
        \EndProcedure
    \end{algorithmic}
\end{algorithm}
\end{minipage}
    \caption{The core procedure for fast Explainable $k$-medians clustering is on the left, with the main subroutine, MedianSplit, on the right. Each node $u$ is set to a single decision tree instruction (with the root $u_0$ having no such instruction), and contains a set of cluster centers $\mathcal{M}(u)$.}
    \label{fig:kmedians_main_alg}
\end{figure}

We note that $B$ could be much larger than $k$ or even $n$, and as a result the above algorithm could take arbitrarily long. In this section, we show how to modify the algorithm, which will allow us to create the decision tree in $O(kd \log^2 k)$ time given the set of clusters $\bo{\mu_1}, \dots, \bo{\mu_k}$, without sacrificing the cost of clustering. Moreover, we show that verification of the explainable clustering being accurate can be done in $O(n(d+H_T))$ time, where $H_T \le k$ is the height of the decision tree $T$.

To create the decision tree, we first completely ignore all points $x \in \mathcal{X},$ and just focus on the cluster centers $\bo{\mu_1}, \dots, \bo{\mu_k}$. At the beginning, we have a node $u_0$ and a cell $c = B'(u_0)$ which is the smallest axis-parallel box that contains all cluster centers. In other words, $c = [a_1, b_1] \times \cdots \times [a_d, b_d]$, where for each dimension $r$, $a_r = \min(\mu_{1, r}, \dots, \mu_{k, r})$ and $b_r = \max(\mu_{1, r}, \dots, \mu_{k, r}).$ Now, we choose a dimension $r$ proportional to $R_r := b_r-a_r,$ and then given $r$, choose a point $z$ uniformly at random in $[a_r, b_r]$. This line $\{x_r = z\}$ will be our first splitting line. Since we chose $a_r < z < b_r$, the cluster centers will not be all on the same side of $\{x_r = z\}$, so we have successfully split the clusters into two regions. Now, each of these two sets of cluster centers will correspond to a new cell, where we again create the smallest axis parallel box that contains all cluster centers. For each cell that has $2$ or more points, we keep repeating this procedure until we have $1$ cluster center per cell.

So, in general, for any node $u$ in the tree $T$ with $|u| \ge 2$, i.e., with more than $1$ cluster center, we choose the a line $\{x_r = z\}$ as done above, and make two new nodes, $\mathcal{L}(u)$ (for \emph{left} child) and $\mathcal{R}(u)$  (for \emph{right} child). $\mathcal{L}(u)$ will represent the part of $u$ with $x_r < z$, so $\mathcal{M}(\mathcal{L}(u)) = \mathcal{M}(u) \cap \{\bo{\mu_i}: \mu_{i, r} < z\},$ while $\mathcal{R}(u)$ will represent the part of $u$ with $x_r \ge z$, so $\mathcal{M}(\mathcal{R}(u)) = \mathcal{M}(u) \cap \{\bo{\mu_i}: \mu_{i, r} \ge z\}.$ The full procedure of our explainable $k$-medians clustering is presented in Figure \ref{fig:kmedians_main_alg}.

\medskip

To see why this procedure still provides a $O(\min(\log k \log \log k, d \log^2 d))$-approximation, we just show that for each point $x \in \mathcal{X}$, the distribution of its assigned cluster $s(x)$ is unchanged. This implies that $\BE[\|x-s(x)\|_1]$ is also unchanged, so the same approximation guarantees hold. We note that the overall clustering distribution may not be the same (since there may be differing correlations between where two points $x, x'$ get assigned), but we only need linearity of expectation to show that our procedure is accurate on average.

To see why this is true, for any point $x \in \mathcal{X}$, let $u_j(x)$ represent the node in $T$ of depth $j$ that contains $x$, where $u_0(x) = u_0$ is the root node and if the leaf node containing $x$ has depth $h$, then $u_j(x)$ is defined to be $u_h(x)$ for $j \ge h$. It suffices to show that for any $x \in \mathcal{X}$, the distribution of $u_j(x)$ is the same regardless of whether we create the decision tree from the simple algorithm described in subsection \ref{subsec:kmedians_simplified} or from the faster algorithm described in this subsection. We prove this by induction, with trivial base case $j = 0$ (since $u_0(x)$ is the root node and is deterministic).

Now, let us assume the claim is true for some $j$, and condition on the node $u_j(x)$. If there is only one cluster center in $u_j(x)$, then $u_{j+1}(x) = u_j(x)$ by definition, so it does not matter which algorithm we choose. Otherwise, let $B'_j(x)$ be the smallest box containing all cluster centers in $u_j(x)$. Note that one we have done the splits necessary to create the node $u_j(x)$, the \emph{simplified algorithm} continues randomly picking lines $\{x_r = z\}$, where $r$ is uniformly selected from $[d]$ and $z$ is uniformly selected from $[-B, B]$. The node $u_j(x)$ remains intact until we have found a line splits at least some clusters in $u_j(x)$ from each other. But the random line splitting at least some clusters in $u_j(x)$ from each other is equivalent to choosing $\{x_r = z\}$ where $z \in [a_r, b_r]$ for $a_r = \min_{\bo{\mu_i} \in \mathcal{M}(u_j(x))} \mu_{i, r}$ and $b_r = \max_{\bo{\mu_i} \in \mathcal{M}(u_j(x))} \mu_{i, r}.$ Clearly, picking a random line conditional on this is equivalent to picking a dimension $r$ proportional to $b_r-a_r$, and then picking a random $z \in [a_r, b_r]$. Therefore, the distribution of $u_{j+1}(x)$ \emph{conditioned on} $u_j(x)$ is the same regardless of whether we used the simplified algorithm or the faster algorithm. This completes the induction.


\medskip

Finally, we describe how to implement this efficiently. For each node $u$, let $B'(u) = [a_1(u), b_1(u)] \times \cdots \times [a_d(u), b_d(u)]$. We store these values, and for each node $u$ and each dimension $r \in [d]$, we also store a balanced binary search tree (BBST) of the key-value pairs $(i, \mu_{i, r})$ for each $\bo{\mu_i} \in B'(u)$, where the BBST is sorted based on the values $\mu_{i, r}$. Each node in the BBST also keeps track of the number of total items to its left and to its right. We also keep a size-$k$ array of pointers the map each $i$ to its corresponding node and its location in each of the $d$ BBSTs.

Now, to split the node $u$, it takes $O(d)$ time to pick a random dimension $r \propto (b_r(u)-a_r(u))$ and a uniformly random $z \in [a_r(u), b_r(u)]$. Next, by binary searching on the dimension $r$ BBST corresponding to node $u$, in $O(\log k)$ time we can determine how many cluster centers in $u$ have $\mu_{i, r} < z$ and how many have $\mu_{i, r} > z.$ For whichever set is smaller, we remove all of those points and create a new BBST with those points, in each dimension. This allows us to have a BBST for each node and each dimension, since we have replaced our node $u$ with two new ones. We note that removal of each point in any fixed dimension takes $O(\log k)$ time, since for a dimension $r'$ and a cluster center $\bo{\mu_i}$, we can access $\mu_{i, r'}$ in $O(1)$ time and then binary search in the dimension-$r'$ BBST. (Note that we can break ties in the $\mu_{i, r'}$ by the value of $i$, so removing $(i, \mu_{i, r'})$ indeed takes $O(\log k)$ time.) Therefore, if the new nodes are $\mathcal{L}(u)$ and $\mathcal{R}(u)$, this takes time $O(d \cdot \log k \cdot \min(|\mathcal{L}(u)|, |\mathcal{R}(u)|))$. Finally, we need to compute the dimensions of the new boxes $B'(\mathcal{L}(u))$ and $B'(\mathcal{R}(u))$. This takes time $O(d \cdot \log k),$ since we just find the smallest and largest elements of each of the $2d$ BBSTs.

Overall, the total runtime is
$$\sum_{u \in T} O(d \log k) \cdot \min(|\mathcal{L}(u)|, |\mathcal{R}(u)|).$$
To bound this, we use the fact that $1 + \log \alpha \le (1+\alpha) \log (1+\alpha)$ for all $\alpha \ge 1,$ which implies that $x + x \log x + y \log y \le (x+y) \log (x+y)$ for all $x \le y$ (by setting $y/x = \alpha$). Therefore, $\min(|\mathcal{L}(u)|, |\mathcal{R}(u)|) + |\mathcal{L}(u)| \log |\mathcal{L}(u)| + |\mathcal{R}(u)| \log |\mathcal{R}(u)| \le |u| \log |u|,$ since $|u| = |\mathcal{L}(u)|+|\mathcal{R}(u)|.$ Therefore, by induction, we obtain the bound
$$\sum_{u \in T} \min(|\mathcal{L}(u)|, |\mathcal{R}(u)|) \le O(k \log k),$$
so the overall runtime to create the decision tree for explainable clustering is $O(d \cdot k \cdot \log^2 k)$.

While our algorithm is randomized and only works in expectation, note that, once given the cluster centers $\bo{\mu_1}, \dots, \bo{\mu_k},$ the algorithm runs in time \emph{sublinear} in the full dataset $\mathcal{X}$. Moreover, the algorithm only depends on the cluster centers, which means that we can run this explainable algorithm on an $O(1)$-approximate $k$-medians coreset of $\mathcal{X}$ and obtain the same $O(\log k \log \log k)$-approximation guarantee in expectation.

If we wish to compute the actual clustering cost given the decision, we need to compute $\|x - s(x)\|_1$ for each $x \in \mathcal{X}$, where $s(x)$ is the assigned cluster. However, since we have computed the entire decision tree, for each $x \in \mathcal{X},$ we just follow it down, which takes $O(H_T)$ time, where $H_T$ is the height of the tree, since we just have to check $1$ coordinate at each step. Finally, computing $\|x-s(x)\|_1$ takes $O(d)$ time. Therefore, doing this for all $\mathcal{X}$ takes $O(n(d+H_T)) = O(n(d+k))$ additional time.

\section{Algorithm for Explainable $k$-means Clustering} \label{sec:kmeans}

In this section, we provide an $O(k \log k)$-approximation algorithm for explainable $k$-means clustering.

For each node $u \in T$, we recall the definitions of $B(u), B'(u), a_r(u), b_r(u), R_r(u),$ and $\mathcal{M}(u)$ from Section \ref{sec:notation}.
Next, we define $\mathcal{X}^{cor}(u)$ represent the points $x \in \mathcal{X}$ that are ``correctly classified'' to be in $u$, i.e., $\mathcal{X}^{cor}(u) = \{x \in \mathcal{X}: x \in B(u), c(x) \in \mathcal{M}(u)\}.$ In addition, for a line $x_r = t$ for some fixed $r \in [d]$ and $t \in [a_r, b_r]$, we say that a point $x \in \mathcal{X}^{cor}(u)$ is \emph{misclassified} by $x_r = t$ if this line splits $x$ from $c(x)$.
Finally, for $r \in [d]$ and $t \in [a_r, b_r],$ define 
\begin{equation} \label{eq:furt}
    f_u(r, t) := \min\left(|\{\bo{\mu_i} \in \mathcal{M}(u): \mu_{i, r} \le t\}|, |\{\bo{\mu_i} \in \mathcal{M}(u): \mu_{i, r} \ge t\}|\right).
\end{equation}
In other words, $f_u(r, t)$ is the minimum of the number of cluster centers $\bo{\mu_i}$ in $B(u)$ such that $\mu_{i, r} \ge t$ and the number of cluster centers $\bo{\mu_i}$ in $B(u)$ such that $\mu_{i, r} \le t$.

The main lemma of Dasgupta et al.~\cite{dasgupta2020explainable} used to obtain an $O(k^2)$-approximation algorithm bounds the number of misclassified points at each split of a node $u$. Their performance in the worst case can be poor if the decision tree $T$ has high depth. First, we improve significantly over their main lemma by balancing the number of misclassified points with $f_u(r, t),$ which represents the lopsidedness of the branching of node $u$. We then show how to apply this improved main lemma to obtain an $O(k \log k)$-approximation. Finally, we analyze the algorithm, showing that we can obtain a fast $O(k^2 d)$ time randomized procedure (with no dependence on the size of the total dataset $\mathcal{X}$), as well as a slower but deterministic $O(kd \cdot n \log n)$ time procedure.

\subsection{Main Lemma} \label{subsec:kmeans_main}

We prove the following main lemma. This lemma improves over Lemma 5.7 in~\cite{dasgupta2020explainable} for the $k$-means case, which was the main technical lemma in the $O(k^2)$-approximation algorithm by~\cite{dasgupta2020explainable}.

\begin{lemma} \label{lem:kmeans_main}
    For any node $u$, there exists $r \in [d]$ and $t \in [a_r(u), b_r(u)]$ such that the number of points in $\mathcal{X}^{cor}(u)$ that are misclassified by the splitting line $x_r = t$ is at most
$$15 \log k \cdot f_u(r, t) \cdot \frac{\sum_{x \in \mathcal{X}^{cor}(u)}\|x-c(x)\|_2^2}{\sum_{r = 1}^{d} R_r(u)^2}.$$
\end{lemma}

\begin{proof}
We treat the node $u$ as fixed in this lemma, so for simplicity, we drop the argument $u$ in $a_r, b_r,$ and $R_r$.

We consider the following procedure of selecting a splitting line. First, select each dimension $r \in [d]$ with probability proportional to $R_r^2.$
Next, select a point $t$ uniformly at random in $[a_r, b_r]$ conditioned on
$$|t-\mu_{i, r}| \ge \frac{R_r}{10 \log {k} \cdot f_u(r, t)} \hspace{0.5cm} \text{for all } \bo{\mu_i} \in \mathcal{M}(u).$$
Let $(r, t)$ be a pair where $r \propto R_r^2$ and $t \sim Unif[a_r, b_r]$. Let $E_u(r, t)$ be the number of misclassified points in $\mathcal{X}^{cor}(u)$ by the line $x_r = t$, i.e.,
\begin{equation} \label{eq:Eurt}
    E_u(r, t) = \left\vert\left\{x \in \mathcal{X}^{cor}(u): x_r < t \le c(x)_r \text{ or } c(x)_r < t \le x_r\right\}\right\vert
\end{equation}
Also, let $\mathcal{A}$ be the event (and $1_\mathcal{A}$ be the indicator random variable) that
$$|t-\mu_{i, r}| \ge \frac{R_r}{10 \log {k} \cdot f_u(r, t)} \hspace{0.5cm} \text{for all } \bo{\mu_i} \in \mathcal{M}(u).$$

First, we note that $\BP(\mathcal{A}) \ge 1/3$. To see why, it suffices to show that conditioned on choosing any fixed dimension $r$, the probability of $\mathcal{A}$ not occurring for a random $t \in [a_r, b_r]$ is at most $2/3$. Let $k' := |\mathcal{M}(u)| \le k$, and let $x_1 \le x_2 \le \dots \le x_{k'}$ be the $r$th coordinates of the points in $\mathcal{M}(u)$ in sorted order. Note that $x_1 = a_r$ and $x_{k'} = b_r.$ Now, if $t \in [x_i, x_{i+1}]$ and $\mathcal{A}$ does not occur, then either $t \in [x_i, x_i + \frac{R_r}{10 \log k \cdot \min(i, k'-i)}]$ or $t \in [x_{i+1}-\frac{R_r}{10 \log k \cdot \min(i, k'-i)}, x_{i+1}]$. Therefore, since $b_r-a_r = R_r$, the probability of $\mathcal{A}$ not occurring conditioned on $r$ is at most
$$\frac{1}{R_r} \cdot \left(\sum_{i = 1}^{k'-1} 2 \cdot \frac{R_r}{10 \log k \cdot \min(i, k'-i)}\right) = \frac{1}{5 \log k} \cdot \sum_{i = 1}^{k'-1} \frac{1}{\min(i,k'-i)}\le \frac{2}{5 \log k} \cdot \sum_{i = 1}^{\lfloor k'/2 \rfloor} \frac{1}{i} \le \frac{2}{3},$$
assuming that $2 \le k' \le k.$

To prove the lemma, it clearly suffices to show that
$$\BE_{r, t}\left[\frac{E_u(r, t)}{f_u(r, t)}\bigg\vert \mathcal{A}\right] \le 15 \log k \cdot \frac{\sum_{x \in \mathcal{X}^{cor}(u)}\|x-c(x)\|_2^2}{\sum_{r = 1}^{d} R_r^2}.$$
Since $\BP(\mathcal{A}) \ge 1/3$, we will just bound $\BE\left[\frac{E_u(r, t)}{f_u(r, t)} \cdot 1_\mathcal{A}\right],$ since 
\begin{equation} \label{eq:indicator_A}
    \BE\left[\frac{E_u(r, t)}{f_u(r, t)}\bigg\vert \mathcal{A}\right] = \frac{\BE\left[\frac{E_u(r, t)}{f_u(r, t)}\cdot 1_{\mathcal{A}}\right]}{\BP(\mathcal{A})} \le 3 \cdot \BE\left[\frac{E_u(r, t)}{f_u(r, t)}\cdot 1_{\mathcal{A}}\right].
\end{equation}

Note that we can write
\begin{equation} \label{eq:expression_integral}
    \BE\left[\frac{E_u(r, t)}{f_u(r, t)} \cdot 1_\mathcal{A}\right] = \frac{1}{\sum_{r = 1}^{d} R_r^2} \cdot \sum_{r = 1}^{d} \int_{a_r}^{b_r} R_r \cdot \frac{E_u(r, t)}{f_u(r, t)} \cdot 1_\mathcal{A} \mathop{dt}.
\end{equation}
Now, if we let $1_{x, t, r}$ be the indicator random variable that $x, c(x)$ are on opposite sides of the line $x_r = t,$ then we can write $E_u(r, t)$ as a sum of indicator variables: $E_u(r, t) = \sum_{x \in \mathcal{X}^{cor}(u)} 1_{x, t, r}.$ Therefore, by Equation \eqref{eq:expression_integral}, we have that
\begin{equation} \label{eq:expression_integral_2}
    \BE\left[\frac{E_u(r, t)}{f_u(r, t)} \cdot 1_\mathcal{A}\right] = \frac{1}{\sum_{r = 1}^{d} R_r^2} \cdot \sum_{r = 1}^{d} \sum_{x \in \mathcal{X}^{cor}(u)} \int_{a_r}^{b_r} R_r \cdot \frac{1_{x, t, r} \cdot 1_\mathcal{A}}{f_u(r, t)} \mathop{dt}.
\end{equation}

Note that $1_{x, t, r} = 1$ if and only if $t$ is between $x_r$ and $c(x)_r$, and $1_{\mathcal{A}} = 1$ only if $|t-c(x)_r| \ge\frac{R_r}{10 \log k \cdot f_u(r, t)}$, which means that $10 \log k \cdot |t-c(x)_r| \ge \frac{R_r}{f_u(r, t)} \cdot 1_{\mathcal{A}}$. Therefore,
\begin{equation} \label{eq:integral_evaluation}
    \int_{a_r}^{b_r} \frac{R_r \cdot 1_{x, t, r} \cdot 1_{\mathcal{A}}}{f_u(r, t)} \mathop{dt} \le 10 \log k \cdot \int_{\min(x_r, c(x)_r)}^{\max(x_r, c(x)_r)} |t-c(x)_r| \mathop{dt} = 5 \log k \cdot (x_r-c(x)_r)^2,
\end{equation}
so by combining Equations \eqref{eq:indicator_A}, \eqref{eq:expression_integral_2}, and \eqref{eq:integral_evaluation}, we obtain
\begin{align}
\BE\left[\frac{E_u(r, t)}{f_u(r, t)} \bigg\vert 1_\mathcal{A}\right] &\le 3 \cdot \BE\left[\frac{E_u(r, t)}{f_u(r, t)} \cdot 1_\mathcal{A}\right] \nonumber \\
&= 3 \cdot \frac{1}{\sum_{r = 1}^{d} R_r^2} \cdot \sum_{r = 1}^{d} \sum_{x \in \mathcal{X}^{cor}(u)} \int_{a_r}^{b_r} R_r \cdot \frac{1_{x, t, r} \cdot 1_\mathcal{A}}{f_u(r, t)} \mathop{dt} \nonumber \\
&\le 3 \cdot \frac{1}{\sum_{r= 1}^{d} R_r^2} \cdot \sum_{x \in \mathcal{X}^{cor}(u)} \sum_{r = 1}^{d} 5 \log k \cdot (x_r-c(x)_r)^2 \nonumber \\
&= 15 \log k \cdot \frac{\sum_{x \in \mathcal{X}^{cor}(u)}\|x-c(x)\|_2^2}{\sum_{r = 1}^{d} R_r^2}, \label{eq:kmeans_main}
\end{align}
as desired.
\end{proof}

\subsection{Finishing the Proof} \label{subsec:kmeans_remainder}

Our algorithm structure is similar to the ``IMM'' algorithm as in Dasgupta et al.~\cite{dasgupta2020explainable}. The main difference is that at each step, we do the splitting according to Lemma \ref{lem:kmeans_main} instead of Lemma 5.7 in \cite{dasgupta2020explainable}. Namely, for each node $u$ of size $|u| \ge 2,$ we choose the pair $(r, t)$ where $t \in [a_r(u), b_r(u)]$, 
that minimizes $\frac{E_u(r, t)}{f_u(r, t)}$. By Lemma \ref{lem:kmeans_main}, we know there exists such a point with $\frac{E_u(r, t)}{f_u(r, t)} \le 15 \log k$.

We present the explainable $k$-means algorithm in Figure \ref{fig:kmeans_main_alg}.
To analyze the accuracy of this algorithm, we use the following lemma, due to Dasgupta et al.~\cite{dasgupta2020explainable}.

\begin{figure}
\centering
\begin{minipage}{0.47\textwidth}
\begin{algorithm}[H]
    \caption{Main procedure for explainable $k$-means}
    \begin{algorithmic}[1] 
        \Procedure{K-means}{$u$}
            \State Use standard $k$-means algorithm to find centers $\bo{\mu_1}, \dots, \bo{\mu_k}$
            \State \textbf{Create} tree $T$ with single node $u_0 \leftarrow \emptyset$ with $\mathcal{M}(u_0) = \{\bo{\mu_1}, \dots, \bo{\mu_k}\}$
            \While{$\exists$ leaf $u \in T$ with $|\mathcal{M}(u)| \ge 2$}
                \State \Call{MeanSplit}{u}
            \EndWhile
            \State \textbf{Return} $T$
        \EndProcedure
    \end{algorithmic}
    \label{alg:kmeans_main}
\end{algorithm}
\end{minipage}
\hfill
\begin{minipage}{0.51\textwidth}
\begin{algorithm}[H]
    \caption{Splitting procedure of a node $u$}
    \begin{algorithmic}[1] 
        \Procedure{MeanSplit}{$u$}
            \For{$r = 1$ to $d$}
                \State $a_r = \mathop{\min}\limits_{\bo{\mu_i} \in \mathcal{M}(u)} \mu_{i, r}$
                \State $b_r = \mathop{\max}\limits_{\bo{\mu_i} \in \mathcal{M}(u)} \mu_{i, r}$
            \EndFor
            \State \textbf{Find} pair $r \in [d], t \in (a_r, b_r)$ minimizing $E_u(r, t)/f_u(r, t)$ \Comment{See equations \eqref{eq:furt}, \eqref{eq:Eurt} for definitions of $E_u(r, t), f_u(r, t)$.}
            \State \textbf{Add} left child $\mathcal{L}(u) \leftarrow \{x_r < t\}$ to $u$
            \State $\mathcal{M}(\mathcal{L}(u)) = \mathcal{M}(u) \cap \{\bo{\mu_i}: \mu_{i, r} < t\}$
            \State \textbf{Add} right child $\mathcal{R}(u) \leftarrow \{x_r \ge t\}$ to $u$
            \State $\mathcal{M}(\mathcal{R}(u)) = \mathcal{M}(u) \cap \{\bo{\mu_i}: \mu_{i, r} \ge t\}$
        \EndProcedure
    \end{algorithmic}
\end{algorithm}
\end{minipage}
\caption{The core procedure for fast Explainable $k$-means clustering is on the left, with the main subroutine, MeanSplit, on the right. The MeanSplit procedure here is deterministic, we later show a faster, but randomized procedure in Figure \ref{fig:kmeans_random_alg}.}
\label{fig:kmeans_main_alg}
\end{figure}

\begin{lemma} \cite[Lemma 5.5, Part 2]{dasgupta2020explainable} \label{lem:dasgupta5.5}
    For any node $u$, recall that $B'(u) = [a_1(u), b_1(u)] \times \cdots \times [a_d(u), b_d(u)]$ is the smallest $d$-dimensional box containing all clusters in $\mathcal{M}(u)$. Then, let $C_2(u) = \sum_{i = 1}^{d} (b_i(u)-a_i(u))^2.$ (This is referred to as $\|\bo{\mu}^{L, u} - \bo{\mu}^{R, u}\|_2^2$ in \cite{dasgupta2020explainable}). Then, the $k$-means cost of the tree $T$ satisfies
\[\text{cost}(T) \le 2 \cdot \text{cost}(\bo{\mu_1}, \dots, \bo{\mu_k}) + 2 \cdot \sum_{u \in T} E_u(r, t) C_2(u),\]
    where $E_u(r, t)$ is the number of points in $\mathcal{X}^{cor}(u)$ that are misclassified when splitting the node $u$.
\end{lemma}

To finish the proof, we first note that for any node $u$, $C_2(u) = \sum_{r = 1}^{d} R_r(u)^2$. Thus, by Lemma \ref{lem:kmeans_main}, we have
\[E_u(r, t) C_2(u) \le 15 \log k \cdot f_u(r, t) \cdot \sum_{x \in \mathcal{X}^{cor}(u)} \|x-c(x)\|_2^2,\]
so 
\[\text{cost}(T) \le 2 \cdot \text{cost}(\bo{\mu_1}, \dots, \bo{\mu_k}) + 30 \log k \cdot \sum_{u \in T} f_u(r, t) \cdot \sum_{x \in \mathcal{X}^{cor}(u)} \|x-c(x)\|_2^2,\]
where $f_u(r, t) = \min(|\mathcal{M}(\mathcal{L}(u))|, |\mathcal{M}(\mathcal{R}(u))|),$ where $\mathcal{L}(u), \mathcal{R}(u)$ are the two direct children of the node $u$.
To finish the proof, it suffices to show that for each $x \in X,$ the term $\|x-c(x)\|_2^2$ appears at most $k$ times in the double summation, or equivalently, for any fixed $x \in \mathcal{X},$
\begin{equation} \label{eq:f_u_sum}
    \sum_{u: x \in \mathcal{X}^{cor}(u)} f_u(r, t) \le k.
\end{equation}

To prove Equation \eqref{eq:f_u_sum}, first note that for any node $u$ with children $v, w,$ $|\mathcal{M}(v)|+|\mathcal{M}(w)| = |\mathcal{M}(u)|,$ so $f_u(r, t) \le \min(|\mathcal{M}(u)|-|\mathcal{M}(v)|, |\mathcal{M}(u)|-|\mathcal{M}(w)|).$ Therefore, since the set of nodes $u$ precisely forms a linear path from the root to some node (let this path of nodes be $u_0, u_2, \dots, u_h$, where $u_0$ is the root of the tree $T$, but $u_h$ may not necessarily be a leaf), we have that
\[\sum_{u: x \in \mathcal{X}^{cor}(u)} f_u(r, t) \le (|\mathcal{M}(u_0)|-|\mathcal{M}(u_1)|) + \cdots + (|\mathcal{M}(u_{h-1})|-|\mathcal{M}(u_h)|) + |\mathcal{M}(u_h)| = |\mathcal{M}(u_0)| = k.\]

\subsection{Analyzing the Runtime} \label{subsec:kmeans_runtime}

For the algorithm described in Subsection \ref{subsec:kmeans_remainder}, the runtime can be analyzed in the same way as in Dasgupta et al.~\cite{dasgupta2020explainable}. Namely, for each node $u$ that we wish to split and each dimension $r \in [d]$, we run a sweep line and keep track the number of misclassified points, while also keeping track of the number of cluster centers in $\mathcal{M}(u)$ that are to the left and to the right of the sweep line, respectively. By sorting the points in $\mathcal{M}(u)$ and $\mathcal{X}^{cor}(u)$ in each dimension, and using dynamic programming to keep track of the number of misclassified points, for any node $u$ we can minimize the ratio $\frac{E_u(r, t)}{f_u(r, t)}$ over all $r \in [d]$ and $t \in [a_r(u), b_r(u)]$ in $O(d n \log n)$ time. Overall, doing this for each node in $u$, we get that once we have our centers from a standard $k$-means clustering algorithm, the remaining runtime is $O(kd n \log n)$, which matches that of \cite{dasgupta2020explainable}. We note this algorithm is \emph{deterministic} and always obtains an $O(k \log k)$-approximation.

Finally, as in the $k$-medians algorithm, we note there also exists a sublinear-time, randomized explainable clustering algorithm that only depends on the cluster centers $\bo{\mu_1}, \dots, \bo{\mu_k}$, which may be generated from a non-explainable $k$-means clustering algorithm. Indeed, the proof of Lemma \ref{lem:kmeans_main} tells us that if we sample each coordinate $r \in [d]$ proportional to $R_r^2$ and select $t \sim Unif[a_r(u), b_r(u)]$, and condition the whole thing on the event $\mathcal{A},$ which is that $|t-\mu_{i, r}| \ge R_r/(10 \log k \cdot f_u(r, t))$ for all $\bo{\mu_i} \in \mathcal{M}(u)$, then $\BE[E_u(r, t)/f_u(r, t)] \le O(\log k) \cdot \sum_{x \in \mathcal{X}^{cor}(u)} \|x-c(x)\|_2^2/(\sum_{r = 1}^{d} R_r^2).$ Therefore, for any node $u$ and set of points $\mathcal{M}(u),$ our randomized procedure will compute $B'(u)$, and then sample a random line $\{x_r = t\}$ where $(r, t)$ is drawn proportional to $R_r \cdot 1_{\mathcal{A}}/f_u(r, t)$. (We remark that the proportionality is $R_r$ instead of $R_r^2$ since the $r^{\text{th}}$ dimension of the box also contributes a factor of $R_r$.) When $(r, t)$ was drawn proportional to $R_r \cdot 1_{\mathcal{A}}$, we had that $\BE\left[\frac{E_u(r, t)}{f_u(r, t)}\right] \le 15 \log k \cdot \sum_{x \in \mathcal{X}^{cor}(u)} \|x-c(x)\|_2^2/(\sum_{r = 1}^{d} R_r(u)^2)$ (see Equation \eqref{eq:kmeans_main}), which means that with our new distribution, we have
\begin{equation} \label{eq:kmeans_random}
    \BE\left[E_u(r, t)\right] \le 15 \log k \cdot \frac{\sum_{x \in \mathcal{X}^{cor}(u)} \|x-c(x)\|_2^2}{\sum_{r = 1}^{d} R_r(u)^2} \cdot \BE\left[f_u(r, t)\right].
\end{equation}

Therefore, if we use this randomized procedure to split the node at each point, and recall that $C_2(u) = \sum_{r = 1}^{d} R_r(u)^2$, we have that
\begin{align*}
\BE[\text{cost}(T)] &\le 2 \cdot \text{cost}(\bo{\mu_1}, \dots, \bo{\mu_k}) + 2 \cdot \BE\left[\sum_{u \in T} E_u(r, t) \cdot C_2(u)\right] \\
&\le 2 \cdot \text{cost}(\bo{\mu_1}, \dots, \bo{\mu_k}) + 2 \cdot \BE\left[\sum_{u \in T} 15 \log k \cdot \frac{\sum_{x \in \mathcal{X}^{cor}(u)} \|x-c(x)\|_2^2}{C_2(u)} \cdot f_u(r, t) \cdot C_2(u)\right] \\
&= 2 \cdot \text{cost}(\bo{\mu_1}, \dots, \bo{\mu_k}) + 30 \log k \cdot \BE\left[\sum_{u \in T}\sum_{x \in \mathcal{X}^{cor}(u)} \|x-c(x)\|_2^2 \cdot f_u(r, t)\right] \\
&\le 2 \cdot \text{cost}(\bo{\mu_1}, \dots, \bo{\mu_k}) + 30 k \log k \cdot \left(\sum_{x \in \mathcal{X}} \|x-c(x)\|_2^2\right),
\end{align*}
which means that in expectation, we have an $O(k \log k)$-approximation. Above, the first line follows from Lemma \ref{lem:dasgupta5.5}. The second line follows from Equation \eqref{eq:kmeans_random} and the fact that our expectation of $E_u(r, t)$ is computed after we already know $u$ (so $C_2(u)$ can essentially be treated as a constant when evaluating the expectation for a single $u$). The third line is simple manipulation, and the final line follows from Equation \eqref{eq:f_u_sum}.

Finally, we show how to actually perform this random procedure efficiently in sublinear time. We will not get $O(k\log^2 k \cdot d)$ as in the $k$-medians case, but we obtain a runtime of $O(k^2 d)$, which is still substantially faster than the deterministic $O(dk \cdot n\log n)$ runtime. First, in $O(kd \log k)$ time, we can assume we have the points $\bo{\mu_1}, \dots, \bo{\mu_k}$ sorted in each dimension. Next, for each node $u$ and each dimension $r \in [d]$, we can use the original sorted points to have the points in $\mathcal{M}(u)$ sorted in dimension $r$ in $O(k)$ time per dimension. If the sorted values in dimension $r$ are $x_{1, r}, \dots, x_{k', r}$ where $k' = |\mathcal{M}(u)|$, then we can compute $[x_{i, r} + \frac{R_r}{10 \log k \cdot \min(i, k'-i)}, x_{i+1, r} - \frac{R_r}{10 \log k \cdot \min(i, k'-i)}]$ for each $i \in [k']$ and $r \in [d].$ Recall that we are sampling the pair $(r, t)$ proportional to $R_r \cdot 1_{\mathcal{A}}/f_u(r, t)$, where $\mathcal{A}$ is the event that $t \in [x_{i, r} + \frac{R_r}{10 \log k \cdot \min(i, k'-i)}, x_{i+1, r} - \frac{R_r}{10 \log k \cdot \min(i, k'-i)}]$ for some choice of $i$. Therefore, by explicitly writing out all of the $k$ relevant intervals in each of the $d$ dimensions, one can easily do the sampling in time $O(k d)$ time. Therefore, since we have to perform this for each node $u \in T$, the overall runtime is $O(k^2 d)$.
Moreover, this algorithm only depends on the cluster centers, which means that we can run this algorithm on an $O(1)$-approximate $k$-means coreset of $\mathcal{X}$ and obtain the same $O(k \log k)$-approximation guarantee in expectation.

The full randomized splitting procedure is shown in Figure \ref{fig:kmeans_random_alg}.

\begin{figure}
    \centering
\begin{algorithm}[H]
    \caption{Randomized splitting procedure of a node $u$}
    \begin{algorithmic}[1] 
        \Procedure{MeanSplitRandom}{$u$}
            \For{$r = 1$ to $d$}
                \State $R_r = \mathop{\max}\limits_{\bo{\mu_i} \in \mathcal{M}(u)} \mu_{i, r}-\mathop{\min}\limits_{\bo{\mu_i} \in \mathcal{M}(u)} \mu_{i, r}$
                \For{$i = 1$ to $|\mathcal{M}(u)|$}
                    \State $x_{i, r} = i^{\text{th}}$ coordinate in sorted order among $\{\mu_{j, r}: \bo{\mu_j} \in \mathcal{M}(u)\}$
                \EndFor
            \EndFor
            \State \textbf{Sample} $(r, t)$ proportional to $\frac{R_r}{\min(i, |\mathcal{M}(u)|-i)}$ if $t \in [x_{i, r} + \frac{R_r}{10 \log k \cdot \min(i, |\mathcal{M}(u)|-i)}, x_{i+1, r} - \frac{R_r}{10 \log k \cdot \min(i, |\mathcal{M}(u)|-i)}]$ for some $1 \le i \le |\mathcal{M}(u)|-1$, proportional to $0$ otherwise.
            \State \textbf{Add} left child $\mathcal{L}(u) \leftarrow \{x_r < t\}$ to $u$
            \State $\mathcal{M}(\mathcal{L}(u)) = \mathcal{M}(u) \cap \{\bo{\mu_i}: \mu_{i, r} < t\}$
            \State \textbf{Add} right child $\mathcal{R}(u) \leftarrow \{x_r \ge t\}$ to $u$
            \State $\mathcal{M}(\mathcal{R}(u)) = \mathcal{M}(u) \cap \{\bo{\mu_i}: \mu_{i, r} \ge t\}$
        \EndProcedure
    \end{algorithmic}
\end{algorithm}
    \caption{Randomized procedure for selecting a splitting line of a node $u$. The main $k$-means procedure (Algorithm \ref{alg:kmeans_main}) can be implemented with MeanSplitRandom as opposed to MeanSplit.}
    \label{fig:kmeans_random_alg}
\end{figure}

Finally, if one wishes to verify the explainable clustering solution's cost on the data, one can perform it in the same manner as in Subsection \ref{subsec:kmedians_faster}, which will require $O(n(d+H_T)) = O(n(d+k))$ time (where $H_T$ is the height of the tree).

\section{Algorithm for explainable $2$-means clustering} \label{sec:2means}

In this section, we provide a $3$-approximation algorithm for $2$-means explainable clustering, which improves over the $4$-approximation algorithm of Dasgupta et al.~\cite{dasgupta2020explainable} and matches the lower bound of \cite{dasgupta2020explainable} when the dimension $d$ is not a constant.

Our algorithm will be identical to that of \cite{dasgupta2020explainable}, which essentially tries all possible decision trees. Because $k = 2$, the decision tree only consists of a single threshold line $\{x_r = z\}$, so for each dimension $r$ from $1$ to $d$, the algorithm runs a sweep line to compute the cost of all possible thresholds. This procedure can be made to run in $O(nd^2 + nd \log n)$, and also has the advantage that it obtains the \emph{optimal} explainable clustering algorithm.

However, unlike the analysis of \cite{dasgupta2020explainable}, our analysis is \emph{probabilistic}. Namely, we provide a randomized procedure that finds an explainable clustering that, in expectation, provides a $3$-approximation to $k$-means. This implies that the optimal explainable algorithm is at most a $3$-approximation, so the algorithm of the previous paragraph will find it.

We now proceed with the analysis.
Let $\bo{\mu_1}$ and $\bo{\mu_2}$ represent the optimal cluster centers for $2$-means clustering.
By reflecting and shifting, we may assume WLOG that $\bo{\mu_1} = (0, 0, \dots, 0) \in \BR^d$ and $\bo{\mu_2} = (R_1, R_2, \dots, R_d) \in \BR^d$, where $R_1, \dots, R_d \ge 0$. Next, we will choose a line based on the following procedure.

Let $$F(x) = \begin{cases}0 & x \le 0 \\ 2x^2 & 0 \le x \le 1/2 \\ 1-2(1-x)^2 & 1/2 \le x \le 1 \\ 1 & x \ge 1 \end{cases}$$
represent the PDF of a distribution $\mathcal{D}$ over $\BR$. Note that $\mathcal{D}$ is supported on $[0, 1]$. We choose $i \in [d]$ proportional to $R_i^2$ (call this distribution $\mathcal{P}$), and then choose the line $\{x_i = R_i \cdot a\},$ where $a \sim \mathcal{D}$.

We will show that for every point $x \in \BR^d$ that is closer to $\bo{\mu_1}$ than to $\bo{\mu_2},$ that 
\begin{equation} \label{eq:3_approx_main}
    \frac{\BP_{i \sim \mathcal{P}, a \sim \mathcal{D}} (x_i \le R_i \cdot a) \cdot \|x\|_2^2 + \BP_{i \sim \mathcal{P}, a \sim \mathcal{D}} (x_i \ge R_i \cdot a) \cdot \|\bo{\mu_2} - x\|_2^2}{\|x\|_2^2} \le 3.
\end{equation}
This is sufficient, as it implies that the expectation of the squared Euclidean distance between $x$ and its assigned cluster, in expectation, is at most $3$ times the squared Euclidean distance between $x$ and its true cluster, for any $x$ closer to $\bo{\mu_1}$ than $\bo{\mu_2}$. However, by the symmetry of the distribution $F(x)$, we also get that this is true for any $x$ closer to $\bo{\mu_2}$ than to $\bo{\mu_1}$. Hence, in expectation, our algorithm provides a $3$-approximation.

Equivalently, by subtracting $1$ from Equation \eqref{eq:3_approx_main} and multiplying by $\|x\|_2^2$, it suffices to show that
\[\BP_{i \sim \mathcal{P}, a \sim \mathcal{D}}(x_i \ge R_i \cdot a) \cdot (\|\bo{\mu_2}-x\|_2^2 - \|x\|_2^2) \le 2 \cdot \|x\|_2^2.\]

Let $x = (R_1 \cdot \alpha_1, R_2 \cdot \alpha_2, \dots, R_d \cdot \alpha_d)$, where $\alpha_1, \dots, \alpha_d \in \BR$. Then, $\|x\|_2^2 = \sum_{i = 1}^{d} R_i^2 \alpha_i^2$, and $\|\bo{\mu_2}-x\|_2^2-\|x\|_2^2 = \sum_{i = 1}^{d} R_i^2 (1 - 2 \alpha_i)$. Finally,
$$\BP_{i \sim \mathcal{P}, a \sim \mathcal{D}} (x_i \ge R_i \cdot a) = \BP_{i \sim \mathcal{P}, a \sim \mathcal{D}}(\alpha_i \ge a) = \frac{\sum_{i = 1}^{d} R_i^2 \cdot F(\alpha_i)}{\sum_{i = 1}^{d} R_i^2}.$$

Hence, it suffices to prove the following lemma.

\begin{lemma}
For any nonnegative real numbers $R_1, \dots, R_n$ and real numbers $\alpha_1, \dots, \alpha_n,$
$$\sum_{i = 1}^{d} R_i^2 (1-2 \alpha_i) \cdot \sum_{i = 1}^{d} R_i^2 F(\alpha_i) \le 2 \sum_{i = 1}^{d} R_i^2 \cdot \sum_{i = 1}^{d} R_i^2 \alpha_i^2.$$
\end{lemma}

\begin{proof}
We define the following quantities:
$$R := \sum_{i = 1}^{d} R_i^2, \hspace{0.5cm} w = \sum_{\alpha_i \ge 0} \alpha_i R_i^2, \hspace{0.5cm} x = \sum_{\alpha_i < 0} (-\alpha_i) R_i^2, \hspace{0.5cm} y = \sum_{\alpha_i \ge 0} R_i^2 \alpha_i^2, \hspace{0.5cm} z = \sum_{\alpha_i < 0} \alpha_i^2 R_i^2.$$
First note that $R, w, x, y, z$ are all nonnegative. Also, note that $0 \le F(x) \le 2x^2$ for all $x \in \BR$, so we can define $y' = \sum_{i = 1}^{d} R_i^2 \cdot F(\alpha_i)/2$, and we have that $0 \le y' \le y$.

The lemma is equivalent to proving 
$(R - 2w + 2x) \cdot y' \le R \cdot (y+z),$ or equivalently, that
$$R(y-y') + Rz + 2wy' \ge 2xy'.$$
Since $y \ge y'$, we have that $R(y-y') \ge 0$. Also, by Cauchy-Schwarz, $Rz \ge \left(\sum_{\alpha_i < 0} \alpha_i R_i^2\right)^2 = x^2$. Finally, note that for all $\alpha_i \ge 0$, $\alpha_i \ge F(\alpha_i)/2$ and for $\alpha_i < 0$, $F(\alpha_i) = 0$, so $w \ge y'$. Therefore, we have that
$$R(y-y')+Rz+2wy' \ge 0+x^2+2(y')^2 \ge x^2 + (y')^2 \ge 2xy'.$$
This proves the lemma, which is also sufficient to establish the $3$-approximation.
\end{proof}

\section{Lower Bounds} \label{sec:lowerbounds}

In this section, we prove unconditional lower bounds for explainable clustering, where we recall that we wish for strong approximations with respect to the optimal non-explainable clustering algorithm. First, in Subsection \ref{subsec:lower_kmedians}, we give a counterexample showing that no explainable clustering algorithm can provide a $o(\min(\log k, d))$-approximation for $k$-medians. Next, in Subsection \ref{subsec:lower_kmeans}, we give a counterexample showing that no explainable clustering algorithm can provide a $o(k)$ approximation for $k$-means, even when $d$ is only logarithmic in $k$.
Finally, we show that our $k$-means lower bound also implies an $\Omega(\sqrt{d} \cdot k)$ lower bound for explainable \emph{$k$-center clustering} for $d = \Omega(\log k)$, providing a slight improvement over the lower bound of Laber and Murtinho~\cite{laber2021explainable}.

\subsection{Lower bound for explainable $k$-medians clustering} \label{subsec:lower_kmedians}

In this subsection, we prove an $\Omega(\log k)$-lower bound for any explainable $k$-medians clustering algorithm, even if the dimension is only $d = O(\log k)$. The lower bound of $\Omega(\log k)$ was already known in the case when $d = \text{poly}(k)$, which also provided an $\Omega(\min(\log k, \log d))$-lower bound, but now we have an improved $\Omega(\min(\log k, d))$-lower bound for explainable $k$-medians clustering.

Before we introduce the construction, we note the following lemma about $k$-medians clustering.

\begin{lemma} \label{lem:k_med_objective}
    Let $x_1, \dots, x_n \in \{-1, 1\}^n$ be clustered into sets $S_1, S_2, \dots, S_k$ that partition $[n]$. Then, for any point $x_i \in S_j$ if $|S_j| = 1$, define $c_i = 0$, and otherwise, define $c_i$ as the average $\ell_1$ distance from $x_i$ to the other points in $S_j$. Then, the minimum $k$-medians clustering cost induced by this partition is at least
$$\frac{1}{4} \cdot \sum_{i = 1}^{n} c_i.$$
\end{lemma}

\begin{proof}
    Fix a cluster (assume WLOG $S_1$) and suppose that $\bo{\mu_1}$ is the optimal cluster center for $S_1.$ Then, suppose that $|S_1| = m > 2$. For each $j \in [d]$, define $a_j$ as the number of points $x_i$ for $i \in S_1$ with $j$th coordinate $x_{ij} = 1$, and $b_j = m-a_j$ as the number of such points with $j$th coordinate $x_{ij} = -1.$ Then, if $\mu_{1j}$ is the $j$th coordinate of $\bo{\mu_1},$ then $\sum_{i \in S_1} |\mu_{1j} - x_{ij}| = a_j \cdot |\mu_{1j}-1| + b_j \cdot |\mu_{1j}+1| \ge 2 \cdot \min(a_j, b_j)$. However, for each point $x_i$ for $i \in S_1$ with $j$th coordinate $1$, its average distance from the other points in just the $j$th direction is $2 \cdot \frac{b_j}{k-1},$ and for each such point with $j$th coordinate $-1$, its average distance from the other points in just the $j$th direction is $2 \cdot \frac{a_j}{k-1}.$ Therefore, the sum of these average distances is $2 \cdot \frac{2a_jb_j}{k-1} \le 4 \cdot \frac{\min(a_j, b_j) \cdot k}{k-1},$ and since $k \ge 2,$ this is at most $8 \cdot \min(a_j, b_j).$ So, $\sum_{i \in S_1} |\mu_{1j}-x_{ij}|$ is at least $\frac{1}{4}$ times the sum of the average distances in the $j$th direction. Adding this up over all coordinates $j$, we get that if $|S_1| \ge 2,$ then $\sum_{i \in S_1} \|\bo{\mu_{1}}-x_i\|_1 \ge \frac{1}{4} \cdot \sum_{i \in S_1} c_i$. Also, if $|S_1| = 1,$ then $\sum_{i \in S_1} \|\bo{\mu_1}-x_i\|_1 \ge 0 = \sum_{i \in S_1} c_i$. Therefore, adding over all clusters gives us the desired result.
\end{proof}

Our construction is somewhat similar to that of Dasgupta et al.~\cite{dasgupta2020explainable}, but our analysis of the lower bound will be different. Let $d = 10 \cdot \log k$ and let $\bo{\mu_1}, \dots, \bo{\mu_k}$ be randomly selected points in the Boolean cube $\{-1, 1\}^d$. By a basic application of the Chernoff bound, we have the following result:

\begin{proposition} \label{prop:chernoff_thing}
    For sufficiently large $k$, with probability at least $0.99$, every pair of points $\bo{\mu_i}$ and $\bo{\mu_j}$ differ in at least $\frac{d}{10}$ coordinates.
\end{proposition}

Now, for each cluster center $\bo{\mu_i},$ we let $\mu_{i, j}$ be the point $\bo{\mu_i}^{\oplus e_j},$ i.e., where we negate the $j$th coordinate of $\bo{\mu_i}.$ The total set of points $\mathcal{X}$ in the dataset will be the $\bo{\mu_i}$ along with the $\mu_{i, j}$'s. Since each of the $k$ cluster centers has $d$ points assigned to it besides itself, the total $k$-medians clustering cost is $2 d k$, since $\|\bo{\mu_i}-\mu_{i, j}\|_1 = 2.$ In addition, the total number of points is $n = k \cdot (d+1)$. Finally, as a direct corollary of Proposition \ref{prop:chernoff_thing}, for any two points in $\mathcal{X}$ not in the same true cluster, they differ in at least $\frac{d}{10}-2$ coordinates, so assuming that $d \ge 40$, their $\ell_1$ distance is at least $2 \cdot \left(\frac{d}{10}-2\right) \ge \frac{d}{10}.$

Now, consider any decision tree process that creates $k$ clusters. Note that if we ever use some coordinate $j$ at some node, we may assume that we never use the same coordinate on its descendants, as the $j$th coordinate only takes two values, so no more information can be obtained about the $j$th coordinate afterwards. Now, for any $i$, let $d_i$ represent the depth of the final (leaf) node in the decision tree that contains $\bo{\mu_i}$, where the depth of the root is defined to be $0$. Then, we must have called $d_i$ separate coordinates on the path from the root to the node, which means that $d_i$ of the points $\mu_{i, j}$ have been separated from $\bo{\mu_i}$, as well as from the remaining points $\mu_{i, j'}$ for $j' \neq j$.

Now, we claim the following lemma.

\begin{lemma}
    Suppose that the dimension $d$ is at least $40$, and define $N := \sum_{i=1}^{k} d_i$. Then, the total clustering cost must be at least $\frac{d}{40} \cdot (N-k).$ So, if $N \ge \Omega(k \log k),$ then the clustering is an $\Omega(\log k)$-approximation.
\end{lemma}

\begin{proof}
    We know that at least $N$ of the points $\mu_{i, j}$ have been separated from the remainder of their true clusters. In addition, since there are $k$ clusters at the end, at least $N-k$ of these points are not assigned to be in clusters by themselves, but are assigned in clusters with other points that are at least $\frac{d}{10}-2$ away from them. So, by Lemma \ref{lem:k_med_objective}, the total clustering cost is at least
\[\frac{1}{4} \cdot (N-k) \cdot 2 \left(\frac{d}{10}-2\right) \ge \frac{d}{40} \cdot (N-k). \qedhere\]
\end{proof}

From now on, we may assume that $\sum_{i=1}^{k} d_i \le \frac{k \log_2 k}{4},$ which means that at least $\frac{k}{2}$ of the values $i \in [k]$ have $d_i \le \frac{\log_2 k}{2}.$ However, note that the number of nodes of depth at most $\frac{\log_2 k}{2}$ is $O(2^{(\log_2 k)/2}) = O(\sqrt{k}).$ Therefore, assuming that $k$ is sufficiently large, at least $\frac{k}{3}$ of the centers $\bo{\mu_i}$ are in the same assigned cluster as at least one other cluster center $\bo{\mu_k}$. In addition, exactly $d_i$ of the cluster center's points $\mu_{i, j}$ are in different cells, so at least $d - d_i \ge 0.9 \cdot d$ of the points $\mu_{i, j}$ are in the same cell as $\bo{\mu_i}$. This implies that for any such $i$ and any such $\mu_{i, j}$ in the same assigned cluster as $\bo{\mu_i}$, the average $\ell_1$ distance between $\mu_{i, j}$ and any other point in its assigned cluster is at least $\frac{d}{20},$ since at least $\frac{1}{2}$ of the points in the assigned cluster are of distance at least $2\left(\frac{d}{10} - 2\right) \ge \frac{d}{20}$ from it. Therefore, by Lemma \ref{lem:k_med_objective}, the total clustering cost is at least
$$\frac{1}{4} \cdot \frac{k}{3} \cdot (0.9 \cdot d) \cdot \frac{d}{20} = \frac{3 k d^2}{800} \ge \frac{3kd}{80} \cdot \log k,$$
so again we have an $\Omega(\log k)$ approximation as the optimal clustering cost is $2 kd$.

\subsection{Lower bound for explainable $k$-means clustering} \label{subsec:lower_kmeans}

In this subsection, we prove an $\Omega(k)$-approximation lower bound for any explainable $k$-means clustering algorithm in $d = \Theta(\log k)$-dimensions. This means that if $d = \Omega(\log k)$, the best possible approximation is $\Omega(k)$
Thus, we provide an exponentially stronger lower bound than the $\Omega(\log k)$-lower bound proven by Dasgupta et al.~\cite{dasgupta2020explainable}.

We create $k$ centers as follows. Let $d$ be the dimension (which we will fix later), and let $\pi_1, \dots, \pi_d: [k] \to [k]$ represent independent random permutations of $\{1, 2, \dots, k\}.$ Our $i$th center $\bo{\mu_i}$ will be $(\pi_1(i), \dots, \pi_d(i)).$ Next, for each cluster center $\bo{\mu_i},$ we assign it $2 \cdot d$ points: for each direction $j$, we create a point $x_{i, j}^{+} = \bo{\mu_i} + e_j$ and another point $x_{i, j}^{-} = \bo{\mu_i} - e_j$, where $e_j$ is the identity vector in the $j$th coordinate. Note that each point $x_{i, j}^+$ and $x_{i, j}^-$ is only $1$ away from its closest center $\bo{\mu_i}$ in Euclidean ($\ell_2$) distance. Our dataset $\mathcal{X}$ will be the set of all $x_{i, j}^+$ and $x_{i, j}^-$ points.

We now show that all of the clusters are far apart with high probability.

\begin{lemma} \label{lem:far_apart}
    There exist absolute constants $C, c > 0$ such that if $d \ge C \log k,$ with probability at least $1/2$, all of the points $\bo{\mu_i}$ are at least $c \cdot k \cdot \sqrt{d}$ away from each other in Euclidean distance.
\end{lemma}

\begin{proof}
    Fix $1 \le i < j \le k$. We consider the random variable $X = \|\bo{\mu_i}-\bo{\mu_j}\|_2^2 = \sum_{r = 1}^{k} (\pi_r(i)-\pi_r(j))^2.$ For a fixed coordinate $r$, the random variable $(\pi_r(i)-\pi_r(j))^2$ is bounded in the range $[0, k^2]$. Moreover, it has expectation at least $c_1 k^2$ for some absolute constant $c_1 > 0$, since with probability at least $1/16,$ $\pi_r(i) \ge 3k/4$ and $\pi_r(j) \le k/4$, in which case $(\pi_r(i)-\pi_r(j))^2 \ge k^2/4.$
    
    Now, define $X_r = (\pi_r(i)-\pi_r(j))^2$ and let $X = X_1+\cdots+X_d.$ Since each $X_r$ is independent (since the permutations are drawn independently), and since each $X_r$ is bounded in the range $[0, k^2]$, we have that
$$\BP\left(|X - \BE[X]| \ge t\right) \le \exp\left(-\frac{2t^2}{d \cdot k^4}\right).$$
    In addition, note that $\BE[X] \ge c_1 d k^2$. Therefore, 
$$\BP\left(X \le \frac{c_1}{2} d k^2\right) \le \exp\left(\frac{-2(c_1dk^2/2)^2}{d \cdot k^4}\right) \le \exp\left(\frac{-c_1^2}{2} \cdot d\right).$$
    Since $X = \|\bo{\mu_i}-\bo{\mu_j}\|_2^2,$ the probability that $\|\bo{\mu_i}-\bo{\mu_j}\|_2 \le c k \sqrt{d}$, for $c = \sqrt{c_1/2}$ and $d \ge C \log k$ for $C = \frac{4}{c_1^2}$, is at most $\frac{1}{k^2}$. Therefore, the probability that there exist any $i \neq j$ such that $\|\mu_i-\mu_j\|_2^2 \le c k \sqrt{d}$ is at most $\frac{1}{k^2} \cdot {k \choose 2} \le \frac{1}{2}$ by the union bound.
\end{proof}

Now, suppose that we have picked some cluster centers $\bo{\mu_1}, \dots, \bo{\mu_k}$ as above, satisfying that all of the centers have pairwise distances at least $c k \sqrt{d}$ from each other. We note that the optimal clustering cost is at most $2d k$, since for each cluster center $\bo{\mu_i}$, there are $d$ points $x_{i, j}^+$ and $d$ more points $x_{i, j}^-$, all of distance $1$ from $\bo{\mu_i}$.

However, no matter what decision tree we choose, we must start off by selecting some line $x_r = t$ for some integer $1 \le r \le d$ and real number $1 \le t \le k.$ Let $i = \pi_r^{-1}(\lfloor t \rfloor).$ Then, $\bo{\mu_i}$ has $r$th coordinate equal to $\lfloor t \rfloor,$ which means that $x_{i, r}^-$ and $x_{i, r}^+$ will be assigned to different clusters. However, since there are only $k$ clusters in total, this means that for any explainable clustering algorithm on $\mathcal{X}$, there must exist points $x, y \in \mathcal{X}$ that were originally assigned to two different clusters $i$ and $j$, but now are assigned to the same cluster. By Lemma \ref{lem:far_apart} and the triangle inequality, $\|x-y\|_2 \ge c k \sqrt{d} - 2$, which means that the $k$-means clustering cost of this algorithm must be at least $\Omega(k^2 d)$. However, since the optimal cost is at most $2kd$, no explainable algorithm can do better than a $O(k)$-approximation for $k$-means clustering, as long as $d \ge \Omega(\log k)$.

We also note that this example also shows that no explainable algorithm can perform better than an $O(k \sqrt{d})$-approximation for the \emph{$k$-center clustering} problem if $d = \Omega(\log k)$. This is because the $k$-center cost of this pointset is $O(1)$ (since every point $x_{i, j}^+$ and $x_{i, j}^-$ is within Euclidean distance $1$ of $\bo{\mu_i}$), but we have shown that any explasinable clustering algorithm must send at least one point $x_{i, j}^+$ or $x_{i, j}^-$ to a cluster of distance $c k \sqrt{d}-2 = \Omega(k \sqrt{d})$. This provides a slight improvement over the $\Omega\left(k \sqrt{d} \cdot \frac{\sqrt{\log \log k}}{\log^{1.5} k}\right)$ lower bound obtained by \cite{laber2021explainable} when $d = \Omega(\log k)$.

\section*{Acknowledgments}

We thank Piotr Indyk and Amin Karbasi for constructive discussions.

\appendix

\section{Dasgupta et. al's Algorithm is Suboptimal for $k$-medians} \label{sec:failure}

In this appendix, we establish that the IMM algorithm by Dasgupta et al.~\cite{dasgupta2020explainable} cannot obtain better than an $O(k)$-approximation for $k$-medians. In addition, even the improvement that we make in the Section \ref{sec:kmeans} of greedily selecting a line based on minimizing the number of misclassified points $E_u(r, t)$ divided by $f_u(r, t)$, rather than just minimizing $E_u(r, t)$, is also suboptimal. Our example assumes the dimension $d$ is $\Theta(k)$.

We now present the dataset for which the IMM algorithm, or even our proposed improvement, fails to obtain better than a $k$-approximation.
All of our points and cluster centers will be in the Boolean hypercube $\{0, 1\}^d$, where we set $d = 2 (k-1)$.
Let $\bo{\mu_1} = \textbf{0} = (0, 0, \dots, 0).$ Next, let $\textbf{z} = (\underbrace{0, 0, \dots, 0}_{k-1}, \underbrace{1, 1, \dots, 1}_{k-1})$ be the point with first $k-1$ coordinates $0$ and last $k-1$ coordinates $1$. For each $1 \le i \le k-1$, we define the $(i+1)$th cluster center as $\bo{\mu}_{i+1} = e_i + \textbf{z},$ where $e_i$ is the identity vector on the $i$th coordinate.

Next, our dataset $\mathcal{X}$ will be as follows. First, for each $2 \le i \le k,$ we let there be $3(k-1)$ copies of $\bo{\mu_i}$ in $\mathcal{X}$. In addition, for each coordinate $1 \le j \le k-1,$ we let there be $1$ copy of $e_j$, and for each coordinate $k \le j \le 2(k-1),$ we let there be $2$ copies of $e_j$. Note that the copies of $\bo{\mu_i}$ for each $2 \le i \le k$ will be assigned to cluster center $\bo{\mu_i}$, and the copies of $e_j$ for each $j \in [d]$ will be assigned to cluster center $\bo{\mu_1}$. In addition, there are $n = 3(k-1) \cdot k$ points in $\mathcal{X}$, $3(k-1)$ assigned to each point $\bo{\mu_i}$, and the total $k$-medians clustering cost is $3(k-1),$ since the copies of $e_j$ are contributing $1$ each to the cost, and the remaining points contribute $0$.

The correct strategy would be to first make a split along one of the last $k-1$ dimensions. This would separate $\bo{\mu_1}$ from all other cluster centers, and anything done now will result in an $O(k)$ clustering cost. Unfortunately, IMM will not do so. Rather, IMM will choose one of the first $k-1$ coordinates to split the dataset, as this causes there to only be $1$ misclassified point instead of $2$. If we split based on the $i$th coordinate, then $\bo{\mu_{i+1}}$ will split from the remaining cluster centers. We will then continue to split the others of the first $k-1$ coordinates. Overall, each of the points $e_1, \dots, e_{k-1}$ will be sent to $\bo{\mu_2}, \dots, \bo{\mu_k},$ respectively, incurring a $k$-medians cost of $(k-1) \cdot k = \Omega(k^2)$. Hence, the IMM algorithm provides an $\Omega(k)$-approximation in the worst case.

Even if we use the modification of IMM that minimizes $E_u(r, t)/f_u(r, t)$ at each step, we would still end up with the same algorithm. This is because in each coordinate, there is always exactly one point among $\{\bo{\mu_1}, \dots, \bo{\mu_k}\}$ with a $0$ at that coordinate, or exactly one point among $\{\bo{\mu_1}, \dots, \bo{\mu_k}\}$ with a $1$ at that coordinate. So, any choice of division will always have $f_u(r, t) = 1,$ so the algorithm does not change, and we still get an $\Omega(k)$-approximation.

\end{document}